\documentclass{article}

\usepackage{paper}
\usepackage{algorithm}
\usepackage{algpseudocode}
\usepackage{wrapfig}
\usepackage{thm-restate}
\usepackage{tikz}
\usepackage{caption}
\usepackage{subcaption}

\newcommand{\GF}{\mathbb{GF}}
\newcommand{\BP}{\mathbb P}
\newcommand{\TW}{\mathrm{TW}}

\makeatletter
\newcommand{\TODO}[1]{\textcolor{red}{[TODO\@ifnotempty{#1}{: #1}]}}

\newcommand{\piotr}[1]{\textcolor{blue}{\textbf{Piotr:} #1}}

\usepackage[final,nonatbib]{neurips}

\usepackage[utf8]{inputenc} 
\usepackage[T1]{fontenc}    
\usepackage{hyperref}       
\usepackage{url}            
\usepackage{booktabs}       
\usepackage{amsfonts}       
\usepackage{nicefrac}       
\usepackage{microtype}      
\usepackage{xcolor}         
\usepackage{enumitem}

\title{Exponentially Improving the Complexity of Simulating the Weisfeiler-Lehman Test with Graph Neural Networks}

\author{%
Anders Aamand \\
MIT \\
\texttt{aamand@mit.edu}
\And
Justin Y.\ Chen \\
MIT \\
\texttt{justc@mit.edu}
\And
Piotr Indyk \\
MIT \\
\texttt{indyk@mit.edu}
\And
Shyam Narayanan \\
MIT \\
\texttt{shyamsn@mit.edu}
\And
Ronitt Rubinfeld \\
MIT \\
\texttt{ronitt@mit.edu}
\And
Nicholas Schiefer \\
MIT \\
\texttt{schiefer@mit.edu}
\And
Sandeep Silwal \\
MIT \\
\texttt{silwal@mit.edu}
\And
Tal Wagner\thanks{Work done prior to joining Amazon.} \\
Amazon AWS \\
\texttt{tal.wagner@gmail.com}
}

\begin{document}

\maketitle

\begin{abstract}
Recent work shows that the expressive power of Graph Neural Networks (GNNs) in distinguishing non-isomorphic graphs is exactly the same as that of the Weisfeiler-Lehman (WL) graph test. In particular, they show that the WL test can be simulated by GNNs. However, those simulations involve neural networks for the “combine” function of size polynomial or even exponential in the number of graph nodes $n$, as well as feature vectors of length linear in $n$. 

We present an improved simulation of the WL test on GNNs with {\em exponentially} lower complexity. In particular,  the neural network implementing the  combine function  in each node has only $\polylog (n)$ parameters, and the feature vectors exchanged by the nodes of GNN consists of only $O(\log n)$ bits. We also give logarithmic lower bounds for the feature vector length and the size of the neural networks, showing the (near)-optimality of our construction. 
\end{abstract}

\section{Introduction}\label{sec:intro}

Graph Neural Networks (GNNs) have become a popular tool for machine learning on graph-structured data, with applications in social network prediction~\cite{hamilton2017inductive}, traffic prediction~\cite{yu2018spatio}, recommender systems~\cite{ying2018graph}, drug discovery~\cite{wieder2020compact}, computer vision~\cite{li2019graph,fey2019deep,monti2017geometric,qi2017pointnet}, and combinatorial optimization~\cite{cappart2021combinatorial}.  Standard message passing GNNs use the topology of the input graph to define the network structure: in each step $k$, a node aggregates messages from each of its neighbors and combines them using a function $\phi^{(k)}$, computed by a neural network, to determine its message for the next round. Crucially, the aggregation function must be {\em symmetric}, to ensure that the output of GNNs is invariant under node permutation. This restriction raised questions about how expressive such network architectures are, and in particular what classes of graphs are distinguishable using GNNs.

The seminal works of Xu et al.~\cite{xu2019} and Morris et al.~\cite{morris2019wl} (see also~\cite{grohe2021logic}) showed that GNNs are exactly as powerful in distinguishing graphs as the Weisfeiler-Lehman (WL) test \cite{wlpaper}, also known as \emph{color refinement}. This combinatorial procedure is a necessary but not sufficient test for graph isomorphism. It proceeds in repeated rounds: in each round, a node labels itself with the ``hash'' of the multiset of labels of its neighbors. The aforementioned papers show that (i) GNNs can simulate the WL test and (ii) GNNs can only distinguish those graphs that the WL test determines to be different. This provides a complete characterization of the expressive power of GNN architectures. 

The connection between GNNs and the WL test has spawned a wave of new results studying GNN variants that either match the distinguishing power of the WL test or adopt new methods beyond message passing on the edges of the input graph to overcome this barrier (see the excellent surveys~\cite{morris2021power,morris2021weisfeiler,grohe2021logic} for an overview of this area).
However, the results in the original as well the follow up works mostly focus on {\em qualitative} questions (how expressive GNNs are) as opposed to  {\em quantitative} questions such as the network complexity.  In particular, while Xu et al.~\cite{xu2019} show that there exist GNN architectures that can simulate the WL coloring procedure as long as the aggregation step is injective, they rely on the universal approximation theorem to show that there exists a neural network that can simulate the hash function used in WL. 
As a result, the size of the network could be exponential in the number of nodes~$n$.  In contrast, the construction of Morris et al.~\cite{morris2021power} uses networks of size polynomial in $n$. However,  the weights of the network implementing $\phi^{(k)}$ in their construction depend on the structure of the underlying graph, which
suffices for  {\em node} classification, but is not sufficient for the context of {\em graph} classification.

Overall, the quantitative understanding of the complexity of simulating WL remains an open problem. Indeed, the survey~\cite{grohe2021logic} states {\em ``The size of the GNNs and related parameters like depth and width, which directly affect the complexity of inference and learning, definitely require close attention}.''

\paragraph{Our Results.}

\begin{table}
\caption{Our results compared to prior work on GNNs that simulate the WL test.}
\label{tbl:results}
    \centering
    {\renewcommand{\arraystretch}{1.3}
    \begin{tabular}{lcccc}
    \toprule
         Construction & Message size & Parameters in $\phi^{(k)}$ & Oblivious\textsuperscript{*} & Deterministic\textsuperscript{*} \\
         \midrule
         Xu et al.~\cite{xu2019} & $O(n)$ & $\Omega(2^n)$ & Yes & Yes \\
         Morris et al.~\cite{morris2019wl} & $O(n)$ & $\poly(n)$ & No & Yes \\
         \midrule
         This paper & $O(\log n)$ & $\polylog (n)$ & Yes & No \\
 \bottomrule
    \end{tabular}
    }
\begin{flushleft}
\begin{footnotesize}
\smallskip
\textsuperscript{*} Oblivious designs use the same weights for any input graph. Non-deterministic constructions require some weights to be assigned by random samples drawn from some distribution, and may err with small probability.
\end{footnotesize}
\end{flushleft}
    
    \label{tab:comparison}
\end{table}

The main question addressed in this work is: what is the simplest (in terms of the number of neural network units and message length) GNN capable of simulating the WL test?  Equivalently, at what point does a GNN become so small that it loses its expressive power?

Our main result is a highly efficient construction of a GNN architecture that is capable of simulating the WL test. For graphs with $n$ nodes, it can simulate $\mbox{poly}(n)$ steps of the WL test, such that the neural network implementing $\phi^{(k)}$ in each round has $\polylog (n)$ parameters, and the messages exchanged by the nodes of the GNN in each round consist of $O(\log n)$ bits. 
This offers at least an exponential improvement over the prior bounds obtained in~\cite{xu2019,morris2019wl} (see Table~\ref{tbl:results}), extending the equivalence between the WL test and the expressive power of GNNs to neural networks of reasonable (in fact, quite small) size. Furthermore, our architecture is simple, using vector sum for aggregation and ReLU units for the combine function $\phi^{(k)}$. {Finally, our construction can be generalized to yield a depth-size tradeoff: for any integer $t>0$, we can can construct a neural network of depth $O(t)$ and size $n^{O(1/t)} \polylog (n)$.}


To achieve this result, our construction is randomized, i.e., some weights of the neural networks are selected at random, and the simulation of the WL test is correct with high probability $1-1/\mbox{poly}(n)$. Thus, our construction can be viewed as creating a {\em distribution} over neural networks computing the function $\phi^{(k)}$.\footnote{Note that selecting $\phi^{(k)}$ at random is quite different from random node initialization, e.g., as investigated in~\cite{abboud2021surprising}. In particular, in our model all nodes use {\em the same} function $\phi^{(k)}$ (with the same parameters), without breaking the permutation invariance property of GNNs, as in the standard GNN model.} In particular, this implies that, for each graph, there exists a single neural network implementing $\phi^{(k)}$ that accurately simulates  WL on that graph.
The size of the network is exponentially smaller than in ~\cite{morris2019wl}, although the construction is probabilistic.




We complement this results with two lower bounds for executing a WL iteration. Our first lower bound addresses the {\em communication complexity} of this problem,  and demonstrates that to solve it,  each node must communicate labels that are at least $O(\log n)$ bits long, matching the upper bound achieved by our construction. Our second lower bound addresses the {\em computational complexity}, namely the parameters of the neural network. It shows that if the messages sent between nodes are vectors with entries in $[F]=\{0,1,\dots, F-1\}$, then the network implementing $\phi^{(k)}$ must use $\Omega(\log F)$ ReLU units.

\paragraph{Related work.} The equivalence between the discriminative power of GNNs and the WL test has been shown in the aforementioned works~\cite{xu2019,morris2019wl}. A strengthened version of the theorem of~\cite{xu2019}, where the same combine function $\phi$ is used in all iterations (i.e., $\phi=\phi^{(k)}$ for all $k$) appeared in the survey~\cite{grohe2021logic}. Many works since have studied various representational issues in GNNs; we refer the reader to excellent surveys~
 \cite{grohe2021logic,huang2021short,morris2021weisfeiler,jegelka2022theory}.
In particular,  \cite{loukas2019graph} established connections between GNNs and distributed computing models such as LOCAL and CONGEST, and derived lower bounds for several computational tasks based on this connection. 
\cite{chen2019equivalence} drew a connection between the expressiveness of GNNs in graph isomorphism testing and in function approximation. 
\cite{barcelo2020logical} studied the expressiveness of GNNs in computing Boolean node classifiers, and \cite{geerts2021let} studied the expressiveness of graph convolutional networks (GCNs). 

The emergence of WL as a barrier in GNN expressivity has also led to a flurry of work on enhancing their expressivity by means of more general architectures. These include higher-order GNNs inspired by higher-dimensional analogs of WL \cite{morris2019wl,maron2019provably}, unique node identifiers \cite{loukas2019graph,vignac2020building}, random node initializations \cite{abboud2021surprising,sato2021random}, relational pooling \cite{murphy2019relational}, incorporating additional information on the graph structure  \cite{nguyen2020graph,barcelo2021graph,bouritsas2022improving,cotta2021reconstruction,toenshoff2021graph}, and more. We refer to \cite{morris2021weisfeiler} for a comprehensive survey of this line of work.



\subsection{Preliminaries}\label{sec:prel}

\paragraph{Notation.} For the rest of the paper, we use $\mathcal N(v)$ to denote the neighborhood of $v$ in a graph $G(V,E)$ \emph{including $v$ itself}, and we use $\{\cdot\}$ to denote \emph{multisets} rather than sets.

\paragraph{GNNs.} Let $G = (V,E)$ be a graph with $N$ nodes. GNNs use the graph structure of $G$ to learn node embeddings for all the nodes across multiple iterations. Let $h_v^{(k)}$ denote the embedding vector of node $v \in V$ in the $k$th iteration. The vectors $h_v^{(0)}$ represent the initial node embeddings. In every iteration $k\geq1$, each node $v$ sends its current embedding $h_v^{(k-1)}$ to all its neighbors, and then computes its new embedding $h_v^{k}$ by the equation
\begin{equation}\label{eq:iteration}
    h_v^{(k)} = \phi^{(k)}\left(f\left(  \left \{h_w^{(k-1)} : w \in \mathcal{N}(v)  \right\}\right) , h_v^{(k-1)} \right)
\end{equation}
where $\phi^{(k)}$ is implemented by a neural network with ReLU activations 
(note that $\phi^{(k)}$ may differ across different iterations $k$). 
The function $f$ is called the `aggregate' function, and $\phi^{(k)}$ is called the `combine' function. 
The embeddings $h_v^{(K)}$ at the final iteration $K$ can be used for node classification. For graph classification, they can be aggregated into a graph embedding $h_G$ with a `readout' function, 
\[ h_G = \textup{READOUT}(\{h_v^{(L)} \mid v \in V \}). \]



\paragraph{Weisfeiler-Lehman (WL) Test.} The WL test \cite{wlpaper} is a popular heuristic for the graph isomorphism problem. While the exact complexity of this problem remains unknown \cite{babai2016graph}, the WL test is a powerful heuristic capable of distinguishing a large family of graphs \cite{babai1979canonical}.

The WL test is as follows. Given a graph $G(V,E)$, initially all nodes are given the same fixed label $\ell_v^{(0)}$, say $\ell_v^{(0)}=1$ for all $v\in V$. Then, in every iteration $k\geq1$, each node $v$ is assigned the new label $\ell_v^{(k)}=\text{HASH}(\{\ell_u^{(k-1)}:u\in\mathcal N(v)\})$, where $\text{HASH}$ is a 1-1 mapping (i.e., different multisets of labels are guaranteed to be hashed into distinct new labels).  
Performing this procedure on two graphs $G,G'$, the WL test declares them non-isomorphic if the label multiset $\{\ell_v^{(k)}:v\in V\}$ differs between the graphs at some iteration $k$. Note that the algorithm converges within at most $n$ iterations.

\paragraph{Simulating WL with GNNs.}
A GNN simulates WL deterministically if the node embeddings $h_v^{(k)}$ at each iteration $k$ constitute valid labels $\ell_v^{(k)}$ for the WL test. We also consider randomized simulation, where the GNN's weights are selected at random and we allow some small failure probability where distinct multisets of labels from iteration $k-1$ are hashed into the same label at iteration $k$. This is captured by the next definition. 


\begin{definition}[Successful Iteration of the WL Test]\label{def:one_iteration}
A WL iteration gets existing labels $\{h_v:v\in V\}$ for all nodes, and outputs new labels $\{h_v':v\in V\}$ given by
\[h_v' = \phi\left(f\left(\left\{h_w : w \in \mathcal{N}(v) \right\}\right)\right), \]
for an aggregate function $f$ and neural network $\phi$ with random weights. We say the iteration is \textbf{successful} if for all $v,u \in V$, the following holds:
\begin{itemize}
    \item If $\{h_w : w \in \mathcal{N}(v)\} = \{h_w : w \in \mathcal{N}(u)\}$ then $h_{v}' = h_{u}'$ with probability 1, and
    \item If $\{h_w : w \in \mathcal{N}(v)\} \ne \{h_w : w \in \mathcal{N}(u)\}$ then $h_{v}' \ne h_{u}'$ with probability $1-p$,
\end{itemize}
where the probability is over the choices of the random weights of $\phi$.
\end{definition}
To ensure the WL simulation is successful across $\mathrm{poly}(n)$ iterations and for all pairs of nodes, we can set failure probability $p$ to $1/\mathrm{poly}(n)$ and apply a union bound (see Prop. \ref{prop:union} in Appendix \ref{sec:omitted_intro}).



\subsection{Overview of Our Techniques}\label{sec:tech-overview}
In this section we give an overview of our GNN architectures for simulating WL. 
To explain our ideas in stages, we begin with a simpler construction of a polynomial size GNN. It is far larger than the ultimate polylogarithmic size we are aiming for, but forms a useful intermediate step toward our second and final construction.
\vspace{-2mm}
\paragraph{Construction 1 (\cref{sec:first_construction}).}
Recall that the $k$th WL iteration, for a node $v$, aggregates the labels of its neighbors from the previous iteration, $\mathcal H_v^{(k-1)}:=\{h_w^{(k-1)}:w\in\mathcal N(v)\}$,
and hashes them into a new label $h_v^{(k)}$ for $v$. Our GNNs aggregate by summing, i.e., they sum $\mathcal H_v^{(k-1)}$ into $\mathcal S_v^{(k)}:=\sum_{w\in\mathcal N(v)}h_w^{(k-1)}$, and then hash the sum into $h_v^{(k)}$ using a ReLU neural network of our choice.

If two nodes $u,v$ satisfy $\mathcal H_u^{(k-1)}\neq\mathcal H_v^{(k-1)}$, then WL assigns them distinct labels in iteration $k$, and therefore we want our construction to satisfy $h_u^{(k)}\neq h_v^{(k)}$ with probability at least $1-p$ (as per \cref{def:one_iteration}). This can fail in two places: either due to summing (if $\mathcal S_u^{(k)}=\mathcal S_v^{(k)}$ even though $\mathcal H_u^{(k-1)}\neq\mathcal H_v^{(k-1)}$) or due to hashing (if $h_u^{(k)}=h_v^{(k)}$ even though $\mathcal S_v^{(k)}\neq\mathcal S_v^{(k)}$).
To avoid the first failure mode (summing), we use one-hot encodings for the node labels, meaning that for every node $v$ and iteration $k$, $h_v^{(k)}$ is a one-hot vector in $\{0,1\}^F$ (for $F$ to be determined shortly). This ensures that if $\mathcal H_u^{(k-1)}\neq\mathcal H_v^{(k-1)}$ then $\mathcal S_u^{(k)}\neq\mathcal S_v^{(k)}$ due to the linear independence of the one-hot vectors. To handle the second failure mode (hashing), we use random hashing into $F=O(1/p)$ buckets, rendering the collision probability no more than $p$. That is, we let $h_v^{(k)}=\text{One-Hot}(g(\mathcal S_v^{(k)}))$ where $g:\Z^F\rightarrow\{1,\ldots,F\}$ 
is chosen at random from an (approximately) universal hash family. We use the simple hash function $g(x)=\langle a,x \rangle \; \mathrm{mod} \; F$ where $a$ is random vector. It can be implemented with a ReLU neural network, since the dot product is just a linear layer, and the mod operation can be implemented with constant width and logarithmic depth, by a reduction to the ``triangular wave'' construction due to Telgarsky~\cite{telgarsky2016}  (see \cref{sec:modF}). Finally, we show that turning the index of the hash bucket into a one-hot vector can be done with $O(F)$ ReLU units and constant depth.

Since our desired collision probability is $p=1/\mathrm{poly}(n)$, we set $F=O(1/p)=\mathrm{poly}(n)$. Since our one-hot vectors are of dimension $F$, the resulting GNN has width $F=\mathrm{poly}(n)$ (and depth $O(\log n)$ due to implementing the mod operation with ReLUs), and thus total size $\mathrm{poly}(n)$. 
\vspace{-2mm}
\paragraph{Construction 2 (\cref{sec:best_construction}).}
We now proceed to describe our polylogarithmic size construction. The weak point in the previous construction was the wasteful use of one-hot encoding vectors, which caused the width to be $F$. In the current construction, we still wish to hash into $F$ bins --- that is, to have $F$ distinct possible labels in each iteration --- but we aim to represent them using bitstrings of length $O(\log F)$, thus exponentially improving the width of the network. That is, for every node $v$ and iteration $k$, the label $h_v^{(k-1)}$ would now be a vector in $\{0,1\}^{O(\log F)}$. The challenge is again to avoid the two failure modes above, ensuring that the failure probability does not exceed $p$. Since we cannot use one-hot encoding, we need to devise another method to avoid the first failure mode, i.e., ensure that with high probability $\mathcal S_u^{(k)} \neq \mathcal S_v^{(k)}$ if $\mathcal H_u^{(k-1)}\neq\mathcal H_v^{(k-1)}$.

To this end, suppose for a moment that we had access to a truly random vector $a\in\{0,1\}^F$. Then each node $v$, instead of sending the one-hot encoding of its label $h_v^{(k-1)}$, we could instead send the dot product $\langle a , h_v^{(k-1)}\rangle$, which is the single bit
$a_{h_v^{(k-1)}}$. Each node $u$ thus receives the bits $\left\{a_{h_w^{(k-1)}}:w\in\mathcal{N}(w)\right\}$ from its neighbors and aggregates them into the sum $\sum_{w\in\mathcal N(u)} \langle a , h_w^{(k-1)}\rangle$, which, by linearity, is equal to $\langle a , \mathcal S_u^{(k)}\rangle$ (using the notation from Construction 1). It is easy to observe that if $\mathcal S_u^{(k)}\neq\mathcal S_v^{(k)}$ then $\langle a , \mathcal S_u^{(k)} \rangle \neq \langle a , \mathcal S_v^{(k)} \rangle$ with probability at least $0.5$. Repeating this process $\log F$ independent times decreases the collision probability to the requisite $1/F$. Thus, we can define a new labeling scheme $\bar h_v^{(k)}\in\{0,1\}^{\log F}$  that concatenates $\log F$ dot products with independent random vectors $a^1 \ldots a^{\log F}\in\{0,1\}^F$ as just described, failing at the summing operation with probability at most $1/F$. The second failure mode (hashing) can again be handled as before.

That catch is that, since $a$ has length $F$,  the overall number of parameters in the GNN would again be at least $F$. To avoid this, we appeal to the computational theory of pseudorandomness. The idea is to replace the random vector $a$ with an efficient pseudorandom analog. Note that the above approach goes through even if the probability that $\langle a , \mathcal S_u^{(k)} \rangle \neq \langle a , \mathcal S_v^{(k)} \rangle$ is slightly larger than $0.5$, say $0.5+\epsilon$ for a small constant $\epsilon$. It is well-known in complexity theory that there exist pseudo-random generators, called \emph{$\epsilon$-biased spaces}, that generate vectors $a$ satisfying this property given only $O(\log F + \log(1/\epsilon))$ truly random bits.\footnote{Technically, they guarantee this property only when $\mathcal S_u^{(k)}$ are binary vectors, but Lemma~\ref{lem:eps_biased} shows how to extend this property to general integer vectors as well.} Crucially,  each bit of $a$ can be computed using a threshold circuit with size polynomial in $O(\log F + \log(1/\epsilon))$ and constant depth (\cref{thm:tc0}), which translates to a ReLU neural network with the same parameters (\cref{lmm:thredhold}). Using these generators in our GNN to implement a pseudorandom ($\epsilon$-biased) analog of $a$ yields our final construction.

\section{First Construction: Polynomial-size GNN}\label{sec:first_construction}
Our first construction towards Definition \ref{def:one_iteration} is exponentially larger compared to our final optimized construction of Section \ref{sec:best_construction}. Nevertheless, it is instructive and motivates our optimized construction.

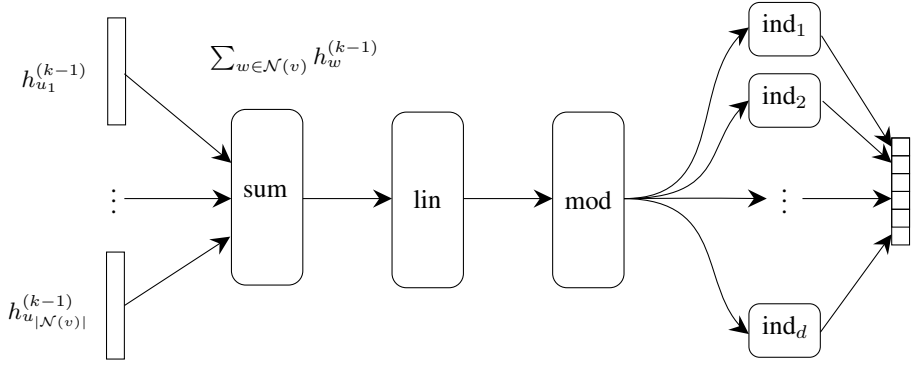
\begin{figure}
\centering
\begin{tikzpicture}[x=0.75pt,y=0.75pt,yscale=-0.9,xscale=0.9]


\draw   (160.67,58.67) -- (150.67,58.67) -- (150.67,118.67) -- (160.67,118.67) -- cycle ;
\draw   (220,118) .. controls (220,113.58) and (223.58,110) .. (228,110) -- (252,110) .. controls (256.42,110) and (260,113.58) .. (260,118) -- (260,200.5) .. controls (260,204.92) and (256.42,208.5) .. (252,208.5) -- (228,208.5) .. controls (223.58,208.5) and (220,204.92) .. (220,200.5) -- cycle ;

\draw   (160,190) -- (150,190) -- (150,250) -- (160,250) -- cycle ;
\draw   (600,126) -- (590,126) -- (590,186) -- (600,186) -- cycle ;
\draw   (600,126) -- (590,126) -- (590,136) -- (600,136) -- cycle ;
\draw   (600,136) -- (590,136) -- (590,146) -- (600,146) -- cycle ;
\draw   (600,146) -- (590,146) -- (590,156) -- (600,156) -- cycle ;
\draw   (600,156) -- (590,156) -- (590,166) -- (600,166) -- cycle ;
\draw   (600,166) -- (590,166) -- (590,176) -- (600,176) -- cycle ;
\draw   (510,56) .. controls (510,52.69) and (512.69,50) .. (516,50) -- (544,50) .. controls (547.31,50) and (550,52.69) .. (550,56) -- (550,74) .. controls (550,77.31) and (547.31,80) .. (544,80) -- (516,80) .. controls (512.69,80) and (510,77.31) .. (510,74) -- cycle ;

\draw   (510,225) .. controls (510,221.69) and (512.69,219) .. (516,219) -- (544,219) .. controls (547.31,219) and (550,221.69) .. (550,225) -- (550,243) .. controls (550,246.31) and (547.31,249) .. (544,249) -- (516,249) .. controls (512.69,249) and (510,246.31) .. (510,243) -- cycle ;

\draw   (510,96) .. controls (510,92.69) and (512.69,90) .. (516,90) -- (544,90) .. controls (547.31,90) and (550,92.69) .. (550,96) -- (550,114) .. controls (550,117.31) and (547.31,120) .. (544,120) -- (516,120) .. controls (512.69,120) and (510,117.31) .. (510,114) -- cycle ;

\draw    (550.67,68.67) -- (588.4,128.46) ;
\draw [shift={(590,131)}, rotate = 237.75] [fill={rgb, 255:red, 0; green, 0; blue, 0 }  ][line width=0.08]  [draw opacity=0] (10.72,-5.15) -- (0,0) -- (10.72,5.15) -- (7.12,0) -- cycle    ;
\draw    (550,235) -- (588.24,182.43) ;
\draw [shift={(590,180)}, rotate = 126.03] [fill={rgb, 255:red, 0; green, 0; blue, 0 }  ][line width=0.08]  [draw opacity=0] (10.72,-5.15) -- (0,0) -- (10.72,5.15) -- (7.12,0) -- cycle    ;
\draw    (551.33,105.33) -- (587.77,138) ;
\draw [shift={(590,140)}, rotate = 221.88] [fill={rgb, 255:red, 0; green, 0; blue, 0 }  ][line width=0.08]  [draw opacity=0] (10.72,-5.15) -- (0,0) -- (10.72,5.15) -- (7.12,0) -- cycle    ;
\draw    (440,160) .. controls (498.8,160.33) and (475.65,93.43) .. (507.95,65.65) ;
\draw [shift={(510,64)}, rotate = 142.96] [fill={rgb, 255:red, 0; green, 0; blue, 0 }  ][line width=0.08]  [draw opacity=0] (10.72,-5.15) -- (0,0) -- (10.72,5.15) -- (7.12,0) -- cycle    ;
\draw    (440,160) .. controls (498.2,160.32) and (484.89,128.66) .. (507.77,106.98) ;
\draw [shift={(510,105)}, rotate = 140.19] [fill={rgb, 255:red, 0; green, 0; blue, 0 }  ][line width=0.08]  [draw opacity=0] (10.72,-5.15) -- (0,0) -- (10.72,5.15) -- (7.12,0) -- cycle    ;
\draw    (440,160) .. controls (498.5,160.33) and (481.43,210.87) .. (507.88,234.26) ;
\draw [shift={(510,236)}, rotate = 217.33] [fill={rgb, 255:red, 0; green, 0; blue, 0 }  ][line width=0.08]  [draw opacity=0] (10.72,-5.15) -- (0,0) -- (10.72,5.15) -- (7.12,0) -- cycle    ;
\draw    (440,160) .. controls (505.33,159.67) and (473.99,159.67) .. (517.26,159.98) ;
\draw [shift={(520,160)}, rotate = 180.4] [fill={rgb, 255:red, 0; green, 0; blue, 0 }  ][line width=0.08]  [draw opacity=0] (10.72,-5.15) -- (0,0) -- (10.72,5.15) -- (7.12,0) -- cycle    ;
\draw    (540,160) -- (587,160) ;
\draw [shift={(590,160)}, rotate = 180] [fill={rgb, 255:red, 0; green, 0; blue, 0 }  ][line width=0.08]  [draw opacity=0] (10.72,-5.15) -- (0,0) -- (10.72,5.15) -- (7.12,0) -- cycle    ;
\draw    (260,160) -- (307,160) ;
\draw [shift={(310,160)}, rotate = 180] [fill={rgb, 255:red, 0; green, 0; blue, 0 }  ][line width=0.08]  [draw opacity=0] (10.72,-5.15) -- (0,0) -- (10.72,5.15) -- (7.12,0) -- cycle    ;
\draw    (160,90) -- (217.7,138.08) ;
\draw [shift={(220,140)}, rotate = 219.81] [fill={rgb, 255:red, 0; green, 0; blue, 0 }  ][line width=0.08]  [draw opacity=0] (10.72,-5.15) -- (0,0) -- (10.72,5.15) -- (7.12,0) -- cycle    ;
\draw    (160,220) -- (216.82,182.97) ;
\draw [shift={(219.33,181.33)}, rotate = 146.91] [fill={rgb, 255:red, 0; green, 0; blue, 0 }  ][line width=0.08]  [draw opacity=0] (10.72,-5.15) -- (0,0) -- (10.72,5.15) -- (7.12,0) -- cycle    ;
\draw    (160,160) -- (217,160) ;
\draw [shift={(220,160)}, rotate = 180] [fill={rgb, 255:red, 0; green, 0; blue, 0 }  ][line width=0.08]  [draw opacity=0] (10.72,-5.15) -- (0,0) -- (10.72,5.15) -- (7.12,0) -- cycle    ;
\draw   (310,119.5) .. controls (310,115.08) and (313.58,111.5) .. (318,111.5) -- (342,111.5) .. controls (346.42,111.5) and (350,115.08) .. (350,119.5) -- (350,202) .. controls (350,206.42) and (346.42,210) .. (342,210) -- (318,210) .. controls (313.58,210) and (310,206.42) .. (310,202) -- cycle ;
\draw    (350,160) -- (397,160) ;
\draw [shift={(400,160)}, rotate = 180] [fill={rgb, 255:red, 0; green, 0; blue, 0 }  ][line width=0.08]  [draw opacity=0] (10.72,-5.15) -- (0,0) -- (10.72,5.15) -- (7.12,0) -- cycle    ;
\draw   (400,119.5) .. controls (400,115.08) and (403.58,111.5) .. (408,111.5) -- (432,111.5) .. controls (436.42,111.5) and (440,115.08) .. (440,119.5) -- (440,202) .. controls (440,206.42) and (436.42,210) .. (432,210) -- (408,210) .. controls (403.58,210) and (400,206.42) .. (400,202) -- cycle ;

\draw (141,82.4) node [anchor=north east] [inner sep=0.75pt]  [font=\footnotesize]  {$h_{u_{1}}^{( k-1)}$};
\draw (142,212.4) node [anchor=north east] [inner sep=0.75pt]  [font=\footnotesize]  {$h_{u_{|\mathcal{N}( v) |}}^{( k-1)}$};
\draw (255.5,81) node  [font=\footnotesize]  {$\sum _{w\in \mathcal{N}( v)} h_{w}^{( k-1)}$};
\draw (530,157) node    {$\vdots $};
\draw (516,55.4) node [anchor=north west][inner sep=0.75pt]    {$\text{ind}_{1}$};
\draw (516,224.4) node [anchor=north west][inner sep=0.75pt]    {$\text{ind}_{d}$};
\draw (516,95.4) node [anchor=north west][inner sep=0.75pt]    {$\text{ind}_{2}$};
\draw (225.23,151) node [anchor=north west][inner sep=0.75pt]   [align=left] {sum};
\draw (154,157) node    {$\vdots $};
\draw (329.92,160.5) node   [align=left] {lin};
\draw (420,160.75) node   [align=left] {mod};
\end{tikzpicture}

\caption{Constructions of our GNNs. Construction 1 (section~\ref{sec:first_construction}): The layer ``sum'' sums the input vectors $\{h_w^{(k-1)}:w\in\mathcal N(v)\}$. ``lin'' is a linear layer that computes the dot product with a random vector and outputs a single scalar. ``mod'' is a neural network that computes the input scalar modulo $F$, using Theorem~\ref{thm:modF}. Each ``ind$_i$'' is a neural network that outputs $1$ if its input equals $i$ and outputs 0 otherwise, using eq.~\eqref{eq:onehot}. The $i$th coordinate of the output vector $h_v^{(k)}$ equals the output of ind$_i$. Construction 2 (section~\ref{sec:best_construction}): Similar, except that now each ind$_i$ is a neural network that gets an input integer $j\in[F]$ and outputs the $j$th coordinate of a vector $a^i$ sampled from an $\epsilon$-biased space, using Theorem~\ref{thm:tc0}. This reduces the total number of units of the network from $\mathrm{poly}(n)$ to $\polylog(n)$. \vspace{-4mm}}
\end{figure}

Let $h_u^{(k)}$ denote the label of a vertex $u$ in the $k$th iteration. For our first construction, we will always maintain the invariant that $h_u^{(k)}$ will be a one-hot encoded vector in $\{0,1\}^{F}$ for all $u \in V$ and all iterations $k$. $F > 2n$ will be a prime which also satisfies $F = O(\poly(n))$. As stated previously, the aggregate function $f$ will just be the sum function. Our construction for the neural network used in the $k$th iteration, $\phi^{(k)}$, will take in the sum of the neighbors labels according to Equation \eqref{eq:iteration} and output a one-hot encoded vector in $\{0,1\}^F$.
 
 \paragraph{Implementation of Neural Network.}
Our construction for $\phi^{(k)}$ is the following: First recall the notation from (simplified) Equation \ref{eq:iteration}:
\[ h_v^{(k)} = \phi^{(k)}\left(f\left(  \left \{h_w^{(k-1)} : w \in \mathcal{N}(v) \right\}\right) \right).\]
 \begin{enumerate}[leftmargin=*]
 \item  For every node $v$, the input to $\phi^{(k)}$ is the sum of feature vectors of neighbors from the prior iteration, $\sum_{w \in \mathcal{N}(v)} h_w^{(k-1)}$, which is returned by the summing function $f$. Given any such input $x \in \Z^F$, $\phi^{(k)}$ computes the inner product of $x$ with a vector $a \in \Z^F$ where each entry of $a$ is a uniformly random integer in $[F] := \{0, \ldots, F-1\}$. 
\item $\phi^{(k)}$ then computes $\langle x, a \rangle \bmod F$. 
 \item Finally, we represent the value of $z = \langle x, a\rangle$ as a one-hot encoded vector in $\Z^F$ where the $z$-th entry is equal to $1$ and all other entries are $0$'s.
 \end{enumerate}
Altogether, $\phi^{(k)}$ can be summarized as: $h_v^{(k)} = \text{One-Hot} \left( \left \langle \sum_{w \in \mathcal{N}(v)} h_w^{(k-1)}, a \right \rangle \bmod F \right)$.
 
Note that we set the initial labels $h_u^{(0)}$ to be the same starting vector for all vertices (any one-hot vector). This matches the WL test which also initializes all nodes with the same initial label. Furthermore, the weights of $\phi^{(k)}$ are independent: the random vector $a$ is sampled independently for each iteration.

\paragraph{Correctness of Construction.}
The following lemma proves that the above construction satisfies the requirement of Definition \ref{def:one_iteration}. Its proof is given in Appendix \ref{sec:omitted_proofs_ub1}.
 
\begin{lemma}\label{lem:first_construction}
Let $\{ h_w^{(k-1)} : w \in \mathcal{N}(v)\}$ and  $\{ h_w^{(k-1)} : w \in \mathcal{N}(u)\}$ denote the multiset of neighborhood labels for vertices $v$ and $u$ respectively. If the multisets are distinct then the labels computed for $v$ and $u$ in the $k$th iteration are the same with probability at most $O(1/F)$. If the multisets are the same then the labels are the same, i.e., the $k$th iteration is successful according to Definition \ref{def:one_iteration}.
\end{lemma}

 
\paragraph{Complexity of the GNN.}
We now evaluate the size complexity of implementing our construction via a neural network $\phi$. Note that Step $1$ of the construction can be done with $1$ layer as it simply involves taking an inner product. The main challenge is to implement the modulo function. We give the following construction in Section \ref{sec:modF} of the appendix.

\begin{restatable}{theorem}{modF}
\label{thm:modF}
Suppose $F = \poly(n)$. There exists an explicit construction of a network which computes modulo $F$ in the domain $ \{0, \ldots, nF \}$ using a ReLU network with $O(\log n)$ hidden units and $O(\log n)$ depth. {More generally, given an integer parameter $t>0$, the function can be computed with $O((Fn)^{O(1/t)} \log n) $ hidden units and $O(t)$ depth.}
\end{restatable}

Directly appealing to the theorem above, we can implement modulo $F$ required in Step $2$ of the construction using a neural network with $O(\log n)$ units, depth $O(\log n)$. In addition, we need only $O(\log n)$ bits to represent the weights.


Finally, Step $3$ of our construction requires outputting a one hot encoding. We can do this by inputting $z$ (the output of Step $2$ of the construction) into $F$ indicator functions, each of which detect if $z$ is equal to a particular integer in $[F]$. Each indicator function can be implemented via $O(1)$ ReLU nodes as follows. Let 
\begin{equation}\label{eq:onehot}
g(x) = \text{ReLU}(\min(-\text{ReLU}(2x-1/2)+1, \text{ReLU}(2x+1)))
\end{equation}
which can be easily implemented as a ReLU network. (Note $\min(a,b) = a+b-\max(a,b)$ and $\max(a,b) = \max(a-b,0)+b = \text{ReLU}(a-b)+b$.) It can be  checked that $g(0) = 1$ and $g(x) = 0$ for all other integers $x \ne 0$ and that $g$ can be implemented with $O(1)$ ReLU function compositions. Thus, Step $3$ of the construction requires $O(1)$ hidden layers and $O(F)$ total hidden units. Altogether we have proven the following result.

\begin{theorem}\label{thm:construction_1}
There exists a construction of a neural network $\phi$ which performs a successful iteration according to Definition \ref{def:one_iteration} with failure probability $p = O(1/|F|)$. $\phi$ has depth $O(\log n)$, $O(F)$ hidden units, and requires $O(\log(nF))$ bits of precision.
 Furthermore, all labels in all iterations are vectors in $\{0,1\}^{F}$.
{More generally, given an integer parameter $t>0$, the function can be computed with $n^{O(1/t)} \polylog (n)$ hidden units and $O(t)$ depth.}
\end{theorem}

\begin{remark}
In the standard WL test, the number of iterations is chosen to be $O(n)$. Thus the right setting of $F$ in Theorem \ref{thm:construction_1} is $F = O(\poly(n))$ which gives us depth $O(\log n)$, $O(\poly(n))$ hidden units, and requires $O(\log n)$ bits of precision in addition to labels in dimension $O(\poly(n))$.
\end{remark}

\section{Second Construction: Polylogarithmic-size GNN via Pseudo-randomness}\label{sec:best_construction}
We now present a more efficient construction of a GNN which simulates the WL test with an exponential improvement in the number of hidden units and label size. To motivate the improvement, we consider Step $3$ of the prior construction which outputs a one-hot encoding. The one-hot encoding was useful as it allowed us to index into a uniformly random vector $a$ (which we then sum over mod $F$ in order to hash the neighborhood's labels). However, this limited us to use feature vectors of a large dimension and required many hidden units to create one-hot vectors. Instead of working with one-hot encodings as an intermediary, we will directly compute the entries of the random vector $a$ as needed. This has two advantages: we can significantly reduce the dimension of the feature vectors as well as reduce the total size of the neural networks used. We accomplish this via using \emph{pseudo-random} vectors whose entries can be generated as needed with a small ReLU neural network (see Corollary \ref{cor:eps_biased}). This allows us to use node labels in dimension $O(\log n)$ as opposed to $O(n)$.


The random vectors we employ have their entries generated from an $\eps$-biased sample space. These are random vectors which are approximately uniform and they have been well-studied in the complexity-theory literature. We recall some definitions below.

\begin{definition}[Bias]
Let $X$ be a probability distribution over $\{0,1\}^m$. The bias of $X$ with respect to a set of indices $I \subseteq \{1, \ldots, m\}$ is defined as $\textup{bias}_I(X) = \left | \BP_{x \sim X}\left[\sum_{i \in I}x_i = 0 \right] -  \BP_{x \sim X}\left[\sum_{i \in I}x_i = 1 \right]  \right |$
where each sum is taken modulo $2$ and the empty sum is defined to be $0$.
\end{definition}

\begin{definition}[$\eps$-biased Sample Space]
A probability distribution over $\{0,1\}^m$ is called an $\eps$-biased sample space if $\textup{bias}_I(X) \le \eps$ holds for all non-empty subsets $I \subseteq  \{1, \ldots, m\}$.
\end{definition}
Note that the uniform distribution has bias $0$. We now state our construction for the neural network used in the $k$th iteration, $\phi^{(k)}$. We recall that $h_u^{(k)}$ denotes the label of a vertex $u$ in the $k$th iteration.

\paragraph{Implementation of Neural Network.}
Our construction for $\phi^{(k)}$ is the following: 

\begin{enumerate}[leftmargin=*]
    \item The input vectors of $\phi^{(k)}$, which are of the form $h_v^{(k-1)}$ for the $k$th iteration, are each assumed to be a feature vector in $\Z^{C \log n}$ for a sufficiently large constant $C$. The output of $\phi^{(k)}$ will also be a feature vector in $\Z^{C \log n}$. Our aggregation function $f$ will again be the summation function.
    \item Let $F$ be a prime of size $\poly(n)$ which is at least $2n$.
    \item For each node $v$, $\phi^{(k)}$ computes $z_v = \langle b, \sum_{w \in \mathcal{N}(v)} h_w^{(k-1)} \rangle \bmod F$ where every entry of $b$ is uniformly random in $\{0, \ldots, F-1\}$. Note $\sum_{w \in \mathcal{N}(v)} h_w^{(k-1)}$ is the output of the aggregation $f$.
    \item Let $a^t \in \{0,1\}^F$ for $t=1, \ldots, C \log n$ be vectors which are independently drawn from an $\eps$-biased sample space for a sufficiently small constant $\eps$. 
    \item The output $h_v^{(k)}$ will be a $C \log n$ dimensional binary vector where the $t$-th coordinate is equal to the $z_v$-th coordinate of the vector $a^t$. In other words,  $h_v^{(k)} = (a^t(z_v))_{t = 1}^{C \log n}$ where $a^t(z_v)$ denotes the $z_v$-th coordinate of $a^t$.
\end{enumerate}

\paragraph{Correctness of Construction.}
We now prove the correctness of our construction. We will refer to $z_v$ computed in Step $3$ of the construction as the \emph{index} of $v$ for the $k$th iteration. To prove the correctness of the above construction, it suffices to prove the lemma below which shows our construction satisfies Definition \ref{def:one_iteration}.

\begin{lemma}\label{lem:gnn_prg}

Let $\{ h_w^{(k-1)} : w \in \mathcal{N}(v)\}$ and  $\{ h_w^{(k-1)} : w \in \mathcal{N}(u)\}$ denote the multiset of neighborhood labels for vertices $v$ and $u$ respectively. If the multisets are distinct then the labels computed for $v$ and $u$ in the $k$th iteration are distinct with probability $1-1/\poly(n)$. If the multisets are the same then the labels are the same, i.e., the $k$th iteration is successful according to Definition \ref{def:one_iteration}.
\end{lemma}

We first need the following auxiliary lemma about $\eps$-biased sample spaces, proven in Section \ref{sec:omitted_proofs_ub2}.

\begin{lemma}\label{lem:eps_biased}
Let $\mathcal{D}$ be a probability distribution over $\{0, 1\}^m$ that is an $\eps$-biased sample space. Then, for any $x, y \in \Z^m$ such that $x \neq y,$ $\BP_{a \sim \mathcal{D}} [\langle a, x \rangle = \langle a, y \rangle] \le \frac{1}{2}+\frac{\eps}{2}$.
\end{lemma}

Note that this lemma is necessary, as we will be computing dot products of $a$ with integer vectors (over integers), not with binary vectors modulo $2$.

We are now ready to prove Lemma \ref{lem:gnn_prg}.

\begin{proof}[Proof of Lemma \ref{lem:gnn_prg}]
Let $x' = \sum_{w \in \mathcal{N}(v)} h_w^{(k-1)}$ denote the input for $v$ and analogously, define $y' = \sum_{w \in \mathcal{N}(u)} h_w^{(k-1)}$ to be the input for $u$. We first show that if $\{ h_w^{(k-1)} : w \in \mathcal{N}(v) \}$ is not equal to (as multisets) $\{ h_w^{(k-1)} : w \in \mathcal{N}(u) \}$
then $x' \ne y'$ with sufficiently large probability. We further consider the case that $k \ge 2$ since for $k=1$ (the first iteration), the statement follows since all node labels are initialized to be the same. Let $z_v'$ be the indices computed in Step $3$ of iteration $k-1$ (which are used to construct the node labels $h_v^{(k-1)}$ in iteration $k-1$). Note that there is a one to one mapping between $z_v'$ and  $h_v^{(k-1)}$. Thus we can assume $\{ z_w' : w \in \mathcal{N}(v) \} \ne  \{z_w' : w \in \mathcal{N}(u)\}$ without loss of generality. 

Let $\tilde{a}^1$ be the first $\eps$-biased vector used in the previous neural network $\phi^{(k-1)}$. Note that the first entry of $x'$ is equal to $\sum_{ w \in \mathcal{N}(v)} \tilde{a}^1(z_w)$, i.e., it is the dot product of $\tilde{a}^1$ with a suitable vector $x$ which is the sum of one-hot encoding of the neighborhood of $v$. The same statement is true for $y'$: the first entry of $y'$ is equal to the dot product of $\tilde{a}^1$ with a vector $y$ which represents the one-hot encoding of the neighborhood of $u$. This is because we computed the index $z_v'$ of every node in the previous iteration, as defined in Step $3$ of the construction, and passed along the coordinate of $\tilde{a}^1$ which corresponds to this computed index. By assumption, we know that $x \ne y$. Therefore by Lemma \ref{lem:eps_biased}, we have that 
$\BP_{\tilde{a}^1}[(x')_1 = (y')_1] = \BP_{\tilde{a}^1}(\langle \tilde{a}^1, x \rangle = \langle \tilde{a}^1, y \rangle) \le 2/3$ for a suitable $\eps$. By independence of vectors $\tilde{a}^t$, it follows that $\BP[x' = y'] \le (2/3)^{C \log n} \le 1/\poly(n)$ for a suitable constant $C>0$. 

We now condition on $x' \ne y'$. Without loss of generality, suppose that their first coordinates, $x'_1$ and $y'_1$, differ. We know $x'_1 \ne y'_1 \bmod F$ since $x'_1 \ne y'_1$, they are both non-negative and bounded by $n$, and $|x'_1 - y'_1| \le O(n)$ whereas $F$ is a prime at least $2n$. It follows that the probability of the event $\langle (x'-y'), b \rangle = 0 \bmod F$ is at most $1/F$. To see this, condition on all the entries of $b$ except $b_1$. Then $(x'-y')_1 \cdot b_1$ must be equal to a specific value modulo $F$ for $\langle (x'-y'), b \rangle = 0 \bmod F$ to hold, as desired. We now condition on this event which equivalently means we condition on $z_v \ne z_u$ (see Step $3$ of the construction).

Now our task is to show that $h_v^{(k)} \ne h_u^{(k)}$ with sufficiently high probability. The first coordinate of $h_v^{(k)}$ is equal to $\langle a^1, e_{z_v} \rangle$ where $a^1$ is the first $\eps$-biased vector considered in Step $4$ of the construction and $e_{z_v}$ is the basis vector in dimension $C \log n$ which has a $1$ entry only in the $z_v$ coordinate and $0$ otherwise. Therefore by Lemma \ref{lem:eps_biased}, we have that , 
$\BP_{a^1}[(h_v^{(k)})_1 = (h_u^{(k)})_1] = \BP_{a^1}[\langle a^1, e_{z_v} \rangle = \langle a^1, e_{z_u} \rangle] \le 2/3$ for a suitable choice of $\eps$ since $e_{z_u}$ and $e_{z_v}$ are distinct. By independence of vectors $a^t$, it follows that $\BP[h_v^{(k)} = h_u^{(k)}] \le (2/3)^{C \log n} \le 1/\poly(n)$  for a suitable constant $C>0$. This exactly means that the node labels of $u$ and $v$ in the next iteration are different with probability at least $1-1/\poly(n)$, as desired. Lastly, it is clear that if the multiset of neighborhood labels of $u$ and $v$ are the same, then $x' = y'$ and it's always the case that $h_v^{(k)} = h_u^{(k)}$.
\end{proof}

\paragraph{Complexity of the GNN.} We now analyze the overall complexity of representing $\phi^{(k)}$ as a ReLU neural network. First we state guarantees on generating $\eps$-based vectors using a ReLU network. The following corollary is proven in Appendix \ref{sec:omitted_proofs_ub2}.


\begin{corollary}\label{cor:eps_biased}
Let $s = O(\log F + \log(1/\eps))$. For every $\eps$ and $F$, there exists an explicit ReLU network $C:\{0,1\}^s \cup [F] \rightarrow \{0,1\}^F$ which takes as input $s$ uniform random bits and an index $i \in F$ and outputs the $i$th coordinate of an $\eps$-biased vector in $\{0,1\}^F$. $C$ uses $O(\log F)$ bits of precision and has $\poly(s)$ hidden units.
{More generally, given an integer parameter $t>0$, the function can be computed with $n^{O(1/t)} \polylog (n)$ hidden units and $O(t)$ depth.}
\end{corollary}

We can now analyze the complexity of our construction. The complexity can be computed by analyzing each step of the construction separately as follows:
\begin{enumerate}[leftmargin=*]
    \item For every node $v$, the sum of feature vectors of neighbors from the prior iteration, $\sum_{w \in \mathcal{N}(v)} h_w^{(k-1)}$, is returned by the aggregate function $f$.
    \item The inner product with the random vector $b$ in Step $3$ of the construction can be computed using one layer of the network. Then computing modulo $F$ can be constructed via Theorem \ref{thm:modF}.
    \item Given the inner product value $z_v$ which is the output of Step $3$ of the construction, we compute all of the $O(\log n)$ coordinates of $h_v^{(k)}$ in parallel. We recall that each coordinate of $h_v^{(k)}$ is indexing onto $O(\log n)$ $\eps$-biased random vectors and we use the \emph{same} index for all vectors, namely the $z_v$-th index. This can be done as follows. We first have $O(\log n)$ edges fanning-out from the node which computes $z_v$. For all $t = 1,\ldots, O(\log n)$, the other endpoint of the $t$-th fan-out edge computes the value $a^t(z_v)$ where $a^t$ is the $t$-th $\eps$-biased vector as stated in Steps 4 and 5 of the construction. This can be done by appealing to the construction guaranteed by Corollary \ref{cor:eps_biased}. The result of this computation is exactly $h_v^{(k)}$.
\end{enumerate}
Altogether, we have proven the following theorem.

\begin{theorem}\label{thm:WL-GNN}
There exists a construction of $\phi$ which performs a successful WL iteration according to Definition \ref{def:one_iteration} with $p \le 1/\poly(n)$. $\phi$ has depth $O(\log n)$, $O(\poly(\log n))$ hidden units, and requires $O(\log n)$ bits of precision. All labels in all iterations are binary vectors in $\{0,1\}^{O(\log n)}$.
{More generally, given an integer parameter $t>0$, the function can be computed with $n^{O(1/t)} \polylog (n)$ hidden units and $O(t)$ depth.}
\end{theorem}

\section{Lower Bounds}\label{sec:lower_bound_overview}
We complement our construction with lower bounds on the label size and number of ReLU units required to simulate the WL test. We outline these two lower bounds below and defer the full details to Appendix~\ref{sec:lower_bounds}.

\paragraph{Message Size.} Recall that in our construction, the message (label) size was $O(\log n)$ bits. Via communication complexity, we give a corresponding lower bound. In particular, we construct a graph on which any (randomized) communication protocol which simulates WL as in Definition~\ref{def:one_iteration} must send at least $\Omega(\log n)$ bits along one edge of the graph. As message-passing GNNs are a specific class of communication protocols, this immediately implies that the message sizes must have $\Omega(\log n)$ bits, so our construction is optimal in that respect.

The hard instance is formed by a graph which is a collection of disjoint star subgraphs of sizes ranging from $2$ to $\Theta(\sqrt{n})$. In order to perform a valid WL coloring, each node must essentially learn the size of its subgraph, requiring $\Omega(\log(\sqrt{n})) = \Omega(\log n)$ bits of communication. In addition, this must be done in only $2$ iterations as the depth of each subgraph is $2$, so some node must send $\Omega(\log n)$ bits to its neighbors in a single round. See Appendix~\ref{sec:low-bound-communication} for the full details and proof.

\paragraph{Number of ReLU Units.} In order to show a lower bound on the number of units needed to implement a successful WL iteration, we rely on prior work lower bounding the number of linear regions induced by a ReLU network (for instance~\cite{montufar2014number}). In particular, these works show that ReLU networks induce a partition of the input space into $N$ convex regions (where $N$ is a function of the size of the network) such that the network acts as a linear function restricted to any given region. Using these results, we describe a fixed graph and a distribution over inputs to the neural network $S_u^{(k-1)}$ for all $u \in V$ (sums of the labels from the previous round) which includes $O(F)$ potential special pairs of nodes (where $F$ is defined such that inputs $S_u^{(k-1)} \in [F]^t$ for some $t$).
For each such pair $u, v$, their neighborhoods $N(u), N(v)$ have different multisets of inputs, but both multisets of inputs sum to the same value.
We show that if the number of linear regions is small, $N = o(F)$, then it is relatively likely that $u, v$ will be in the same linear region and thus their sums will collide: $S_u^{(k)} = S_v^{(k)}$ even while their neighborhoods had distinct inputs in the $(k-1)$st round.

This immediately gives a $\Omega(\log F)$ lower bound on the number of ReLU units (and thus number of parameters) with more refined depth/width tradeoffs given in \cref{sec:more-units-results}. Note that $F$ is the size of each coordinate in the \emph{sum} of labels. Even if the labels are binary, $F$ can be as large as $n$, depending on the max degree in the graph, which implies a $\Omega(\log n)$ lower bound on the number of ReLU units. See \cref{sec:low-bound-units} for full details and proof.

\section{Experiments}
To demonstrate the expressivity of our construction, i.e., that our small-sized GNN reliably simulates the WL test, we perform experiments on both synthetic and real world data sets. Common to all of our experiments is that we start with some graph $G=(V,E)$ (either real world or generated with respect to some probability distribution). We then simulate a perfect run of the WL test on $G$ where any two nodes which receive different multisets of labels in iteration $k-1$ get distinct labels in iteration $k$ with probability $1$ as well as a run of our construction from\footnote{Since our goal is to test whether our protocol correctly simulates WL test with small messages, we are not implementing the actual GNNs but instead we are simulating their computation. Further, for simplicity, we replaced the $\eps$-biased sample space with a random string, which guarantees $\eps=0$.}~\Cref{sec:best_construction}.
At any point in time, the node labels induce partitions of $V$ where two nodes are in the same class if they have the same labels. Denote the partitions after $k$-iterations using the perfect simulation and our construction respectively by $\mathcal{P}_k$ and $\mathcal{P}_k'$. Letting $k_0$ be minimal such that $\mathcal{P}_{k_0-1}=\mathcal{P}_{k_0}$ (at which point the WL labels have converged), we consider the implementation using our GNN successful if $\mathcal{P}_k=\mathcal{P}_k'$ for all $k\leq k_0$, i.e., if the the simulation using our implementation induced the same partitions as a perfect runs. For all of our experiments it turned out that $k_0\leq 5$ (see~\cite{Bause2022slowandsteady} for a discussion of this fast convergence).

\begin{figure}[ht]
\centering
\begin{subfigure}[b]{0.3\textwidth}
    \centering
    \includegraphics[width=\textwidth]{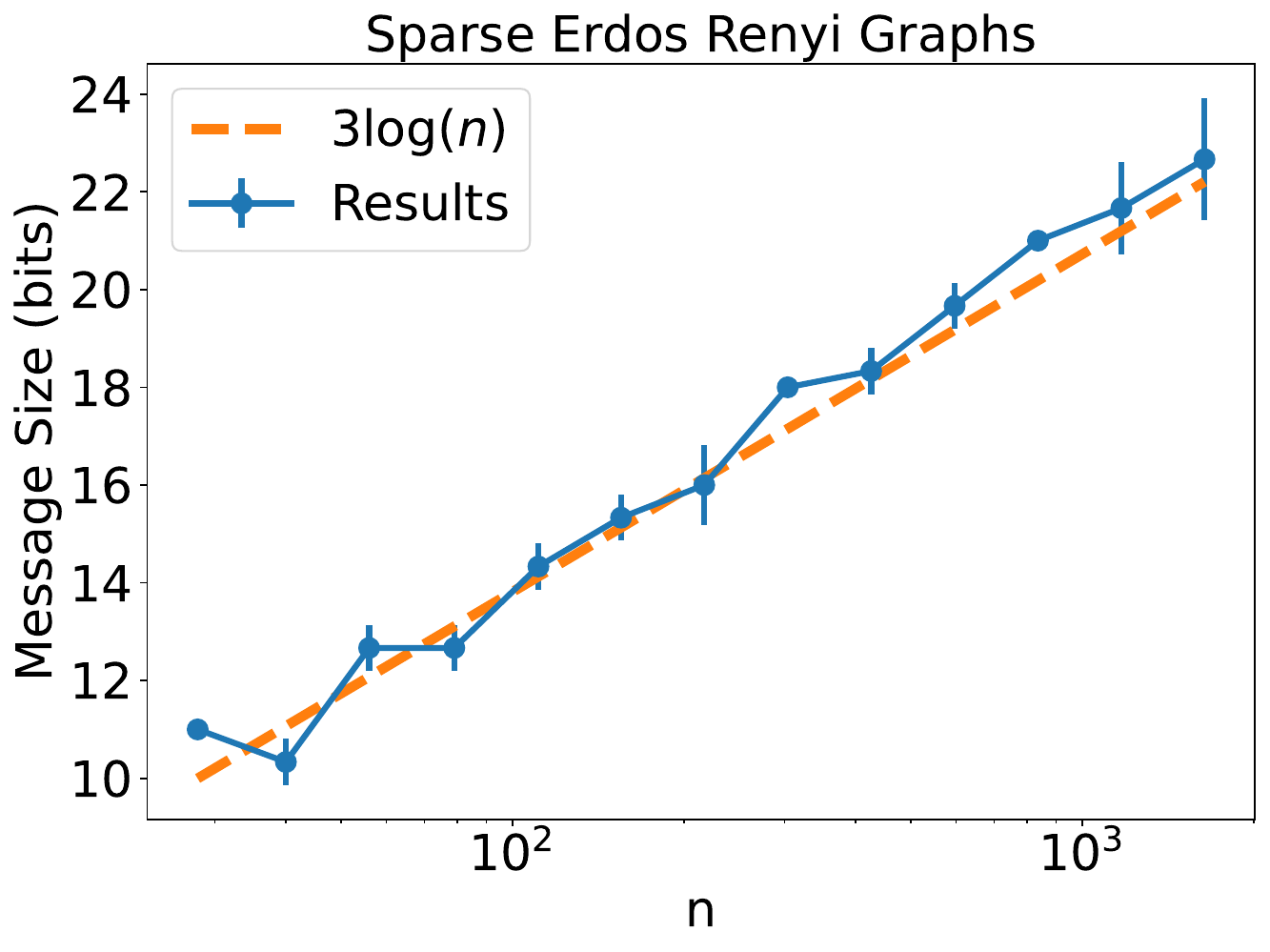}
    \caption{}
    \label{fig:erdos-renyi}
\end{subfigure} \hspace{0.2in}
\begin{subfigure}[b]{0.3\textwidth}
    \centering
    \includegraphics[width=\textwidth]{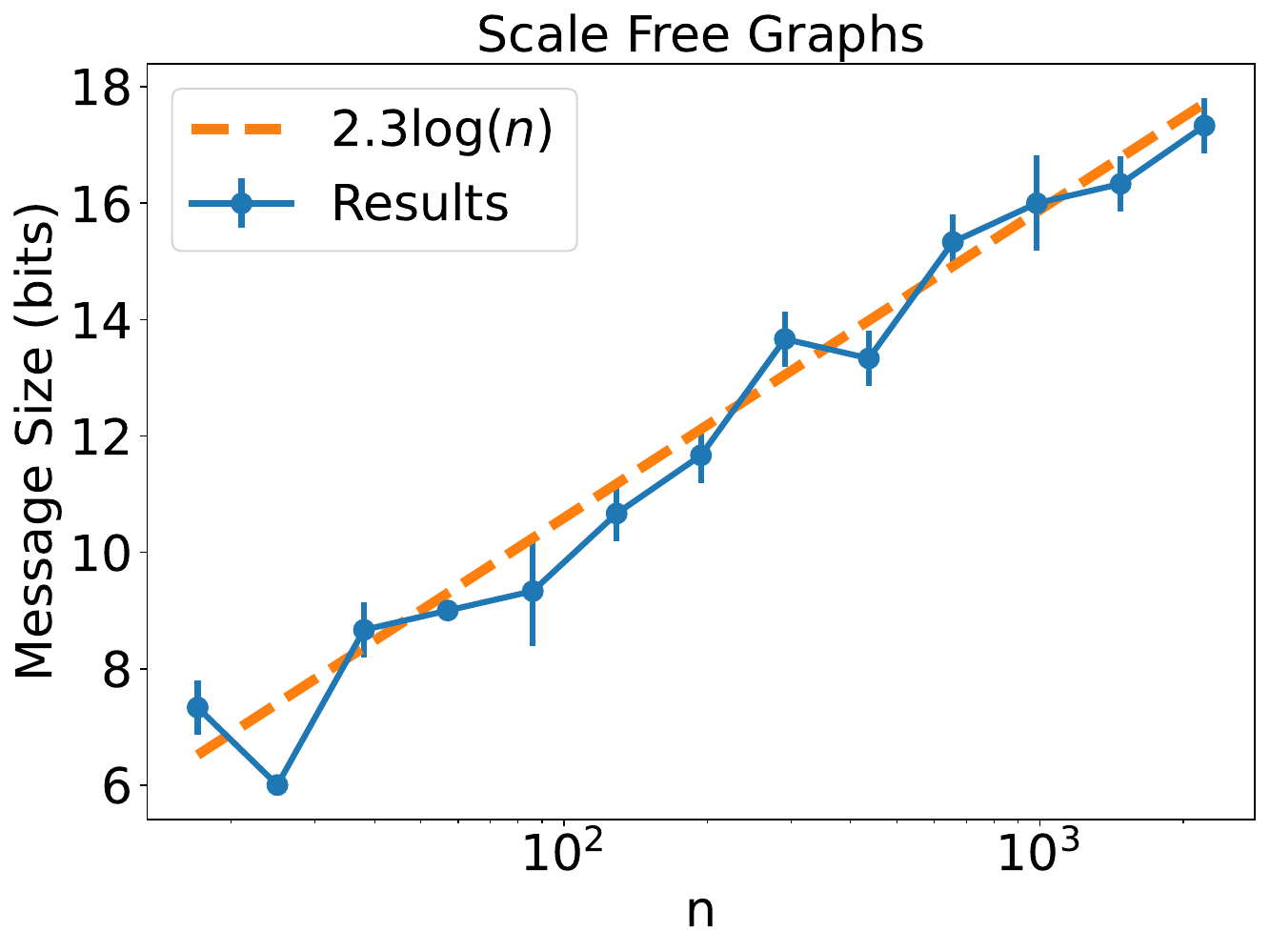}
    \caption{}
    \label{fig:scale-free}
\end{subfigure} \hspace{0.2in}
\begin{subfigure}[b]{0.3\textwidth}
    \centering
    \includegraphics[width=\textwidth]{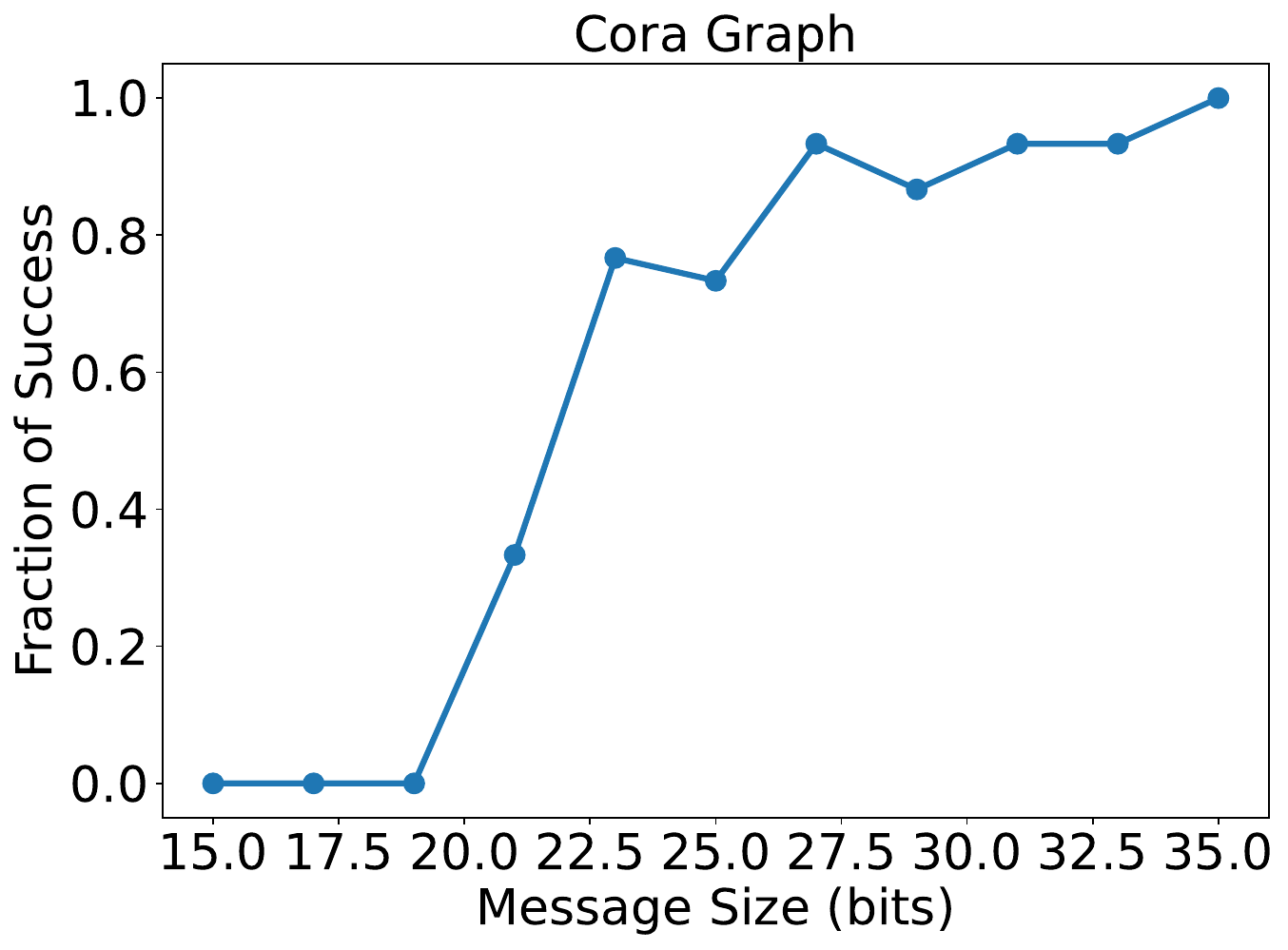}
    \caption{}
    \label{fig:cora}
\end{subfigure}
\caption{Plots of our experimental results on Erd\H{o}s-Rényi graphs, scale free graphs, and the real world Cora graph. The vertical blue lines in (a) and (b) are the empirical standard deviations over the 5 independent samples of random graphs. Note that the $x$-axes in (a) and (b) are logarithmic.}
\label{fig:scale-free-and-random}
\vspace{-0.1 in}
\end{figure}

\paragraph{Sparse Erd\H{o}s-Rényi Graphs} We generated Erd\H{o}s-Rényi random graphs $G(n,p)$ with $p=20/n$ for a varying number of vertices $n$. For each value of $n$, we generated five such graphs and for each of these five graphs, we ran 10 independent trials of our GNN implementation with message sizes $t=1,2,\dots$. Averaging over the five graphs, we report the minimal $t$ such that at least $70\%$ of the 10 iterations successfully simulated the WL test. See~\Cref{fig:erdos-renyi}. The average message size needed to achieve this is approximately $3\log n$ where the logarithmic dependence on $n$ is as predicted theoretically and significantly improves on the linear message size required for prior constructions.

\paragraph{Scale Free Graphs} We generated samples of the scale free graphs from~\cite{Bollobas2003scalefree} with a varying number of vertices $n$ using the implementation from~\cite{SciPyProceedings_11}. Our experiment design was the same as for Erd\H{o}s-Rényi random graphs. See~\Cref{fig:scale-free}.

\paragraph{Cora Graph}
We finally ran experiments on the real world graph Cora\footnote{https://graphsandnetworks.com/the-cora-dataset/} which is the citation network of $n=2708$ scientific publications. We simulated our GNN with varying message lengths, for each message length reporting the fraction of successful runs of $30$ independent trials. See~\Cref{fig:cora} for a plot of the results. We see that with message length $35$, all of the $30$ trials successfully simulated the WL test. 

\paragraph{Acknowledgements}
Anders Aamand is supported by DFF-International Postdoc Grant 0164-00022B from the Independent Research Fund Denmark. This research was also supported by the NSF TRIPODS program
(award DMS-2022448), NSF award CCF-2006664,  Simons Investigator Award, MIT-IBM Watson AI Lab, GIST-
MIT Research Collaboration grant, NSF
Graduate Research Fellowship under Grant No. 1745302, and MathWorks Engineering Fellowship.

\bibliographystyle{alpha}
\bibliography{bib}

\newcommand{\etalchar}[1]{$^{#1}$}
\begin{thebibliography}{MBHSL19}

\bibitem[ACGL21]{abboud2021surprising}
Ralph Abboud, Ismail~Ilkan Ceylan, Martin Grohe, and Thomas Lukasiewicz.
\newblock The surprising power of graph neural networks with random node
  initialization.
\newblock In {\em IJCAI}, 2021.

\bibitem[Bab16]{babai2016graph}
L{\'a}szl{\'o} Babai.
\newblock Graph isomorphism in quasipolynomial time.
\newblock In {\em Proceedings of the forty-eighth annual ACM symposium on
  Theory of Computing}, pages 684--697, 2016.

\bibitem[BBCR03]{Bollobas2003scalefree}
B\'{e}la Bollob\'{a}s, Christian Borgs, Jennifer Chayes, and Oliver Riordan.
\newblock Directed scale-free graphs.
\newblock In {\em Proceedings of the Fourteenth Annual ACM-SIAM Symposium on
  Discrete Algorithms}, SODA '03, page 132–139, USA, 2003. Society for
  Industrial and Applied Mathematics.

\bibitem[BFZB22]{bouritsas2022improving}
Giorgos Bouritsas, Fabrizio Frasca, Stefanos~P Zafeiriou, and Michael
  Bronstein.
\newblock Improving graph neural network expressivity via subgraph isomorphism
  counting.
\newblock {\em IEEE Transactions on Pattern Analysis and Machine Intelligence},
  2022.

\bibitem[BGRR21]{barcelo2021graph}
Pablo Barcel{\'o}, Floris Geerts, Juan Reutter, and Maksimilian Ryschkov.
\newblock Graph neural networks with local graph parameters.
\newblock {\em Advances in Neural Information Processing Systems}, 34, 2021.

\bibitem[BK79]{babai1979canonical}
L{\'a}szl{\'o} Babai and Ludik Kucera.
\newblock Canonical labelling of graphs in linear average time.
\newblock In {\em 20th Annual Symposium on Foundations of Computer Science
  (sfcs 1979)}, pages 39--46. IEEE, 1979.

\bibitem[BK22]{Bause2022slowandsteady}
Franka Bause and Nils~M. Kriege.
\newblock Gradual weisfeiler-leman: Slow and steady wins the race.
\newblock {\em CoRR}, abs/2209.09048, 2022.

\bibitem[BKM{\etalchar{+}}20]{barcelo2020logical}
Pablo Barcel{\'o}, Egor Kostylev, Mikael Monet, Jorge P{\'e}rez, Juan Reutter,
  and Juan-Pablo Silva.
\newblock The logical expressiveness of graph neural networks.
\newblock In {\em 8th International Conference on Learning Representations
  (ICLR 2020)}, 2020.

\bibitem[CCK{\etalchar{+}}21]{cappart2021combinatorial}
Quentin Cappart, Didier Ch{\'e}telat, Elias Khalil, Andrea Lodi, Christopher
  Morris, and Petar Veli{\v{c}}kovi{\'c}.
\newblock Combinatorial optimization and reasoning with graph neural networks.
\newblock {\em IJCAI}, 2021.

\bibitem[CMR21]{cotta2021reconstruction}
Leonardo Cotta, Christopher Morris, and Bruno Ribeiro.
\newblock Reconstruction for powerful graph representations.
\newblock {\em Advances in Neural Information Processing Systems}, 34, 2021.

\bibitem[CVCB19]{chen2019equivalence}
Zhengdao Chen, Soledad Villar, Lei Chen, and Joan Bruna.
\newblock On the equivalence between graph isomorphism testing and function
  approximation with gnns.
\newblock {\em Advances in neural information processing systems}, 32, 2019.

\bibitem[FLM{\etalchar{+}}19]{fey2019deep}
Matthias Fey, Jan~E Lenssen, Christopher Morris, Jonathan Masci, and Nils~M
  Kriege.
\newblock Deep graph matching consensus.
\newblock In {\em International Conference on Learning Representations}, 2019.

\bibitem[GMP21]{geerts2021let}
Floris Geerts, Filip Mazowiecki, and Guillermo Perez.
\newblock Let’s agree to degree: Comparing graph convolutional networks in
  the message-passing framework.
\newblock In {\em International Conference on Machine Learning}, pages
  3640--3649. PMLR, 2021.

\bibitem[Gro21]{grohe2021logic}
Martin Grohe.
\newblock The logic of graph neural networks.
\newblock In {\em 2021 36th Annual ACM/IEEE Symposium on Logic in Computer
  Science (LICS)}, pages 1--17. IEEE Computer Society, 2021.

\bibitem[HSS08]{SciPyProceedings_11}
Aric~A. Hagberg, Daniel~A. Schult, and Pieter~J. Swart.
\newblock Exploring network structure, dynamics, and function using networkx.
\newblock In Ga\"el Varoquaux, Travis Vaught, and Jarrod Millman, editors, {\em
  Proceedings of the 7th Python in Science Conference}, pages 11 -- 15,
  Pasadena, CA USA, 2008.

\bibitem[HV06]{healy2006constant}
Alexander Healy and Emanuele Viola.
\newblock Constant-depth circuits for arithmetic in finite fields of
  characteristic two.
\newblock In {\em Annual Symposium on Theoretical Aspects of Computer Science},
  pages 672--683. Springer, 2006.

\bibitem[HV21]{huang2021short}
Ningyuan~Teresa Huang and Soledad Villar.
\newblock A short tutorial on the weisfeiler-lehman test and its variants.
\newblock In {\em ICASSP 2021-2021 IEEE International Conference on Acoustics,
  Speech and Signal Processing (ICASSP)}, pages 8533--8537. IEEE, 2021.

\bibitem[HYL17]{hamilton2017inductive}
Will Hamilton, Zhitao Ying, and Jure Leskovec.
\newblock Inductive representation learning on large graphs.
\newblock {\em Advances in neural information processing systems}, 30, 2017.

\bibitem[Jeg22]{jegelka2022theory}
Stefanie Jegelka.
\newblock Theory of graph neural networks: Representation and learning.
\newblock {\em arXiv preprint arXiv:2204.07697}, 2022.

\bibitem[LGD{\etalchar{+}}19]{li2019graph}
Yujia Li, Chenjie Gu, Thomas Dullien, Oriol Vinyals, and Pushmeet Kohli.
\newblock Graph matching networks for learning the similarity of graph
  structured objects.
\newblock In {\em International conference on machine learning}, pages
  3835--3845. PMLR, 2019.

\bibitem[Lou19]{loukas2019graph}
Andreas Loukas.
\newblock What graph neural networks cannot learn: depth vs width.
\newblock In {\em International Conference on Learning Representations}, 2019.

\bibitem[MBHSL19]{maron2019provably}
Haggai Maron, Heli Ben-Hamu, Hadar Serviansky, and Yaron Lipman.
\newblock Provably powerful graph networks.
\newblock {\em Advances in neural information processing systems}, 32, 2019.

\bibitem[MBM{\etalchar{+}}17]{monti2017geometric}
Federico Monti, Davide Boscaini, Jonathan Masci, Emanuele Rodola, Jan Svoboda,
  and Michael~M Bronstein.
\newblock Geometric deep learning on graphs and manifolds using mixture model
  cnns.
\newblock In {\em Proceedings of the IEEE conference on computer vision and
  pattern recognition}, pages 5115--5124, 2017.

\bibitem[MFK21]{morris2021power}
Christopher Morris, Matthias Fey, and Nils~M Kriege.
\newblock The power of the weisfeiler-leman algorithm for machine learning with
  graphs.
\newblock {\em IJCAI}, 2021.

\bibitem[MLM{\etalchar{+}}21]{morris2021weisfeiler}
Christopher Morris, Yaron Lipman, Haggai Maron, Bastian Rieck, Nils~M Kriege,
  Martin Grohe, Matthias Fey, and Karsten Borgwardt.
\newblock Weisfeiler and leman go machine learning: The story so far.
\newblock {\em arXiv preprint arXiv:2112.09992}, 2021.

\bibitem[Mon17]{montufar2017notes}
Guido Mont{\'u}far.
\newblock Notes on the number of linear regions of deep neural networks.
\newblock 2017.

\bibitem[MPCB14]{montufar2014number}
Guido~F Montufar, Razvan Pascanu, Kyunghyun Cho, and Yoshua Bengio.
\newblock On the number of linear regions of deep neural networks.
\newblock {\em Advances in neural information processing systems}, 27, 2014.

\bibitem[MRF{\etalchar{+}}19]{morris2019wl}
Christopher Morris, Martin Ritzert, Matthias Fey, William~L. Hamilton, Jan~Eric
  Lenssen, Gaurav Rattan, and Martin Grohe.
\newblock {Weisfeiler and Leman Go Neural: Higher-Order Graph Neural Networks}.
\newblock AAAI'19/IAAI'19/EAAI'19. AAAI Press, 2019.

\bibitem[MSRR19]{murphy2019relational}
Ryan Murphy, Balasubramaniam Srinivasan, Vinayak Rao, and Bruno Ribeiro.
\newblock Relational pooling for graph representations.
\newblock In {\em International Conference on Machine Learning}, pages
  4663--4673. PMLR, 2019.

\bibitem[NM20]{nguyen2020graph}
Hoang Nguyen and Takanori Maehara.
\newblock Graph homomorphism convolution.
\newblock In {\em International Conference on Machine Learning}, pages
  7306--7316. PMLR, 2020.

\bibitem[QSMG17]{qi2017pointnet}
Charles~R Qi, Hao Su, Kaichun Mo, and Leonidas~J Guibas.
\newblock Pointnet: Deep learning on point sets for 3d classification and
  segmentation.
\newblock In {\em Proceedings of the IEEE conference on computer vision and
  pattern recognition}, pages 652--660, 2017.

\bibitem[RPK{\etalchar{+}}17]{raghu2017expressive}
Maithra Raghu, Ben Poole, Jon Kleinberg, Surya Ganguli, and Jascha
  Sohl-Dickstein.
\newblock On the expressive power of deep neural networks.
\newblock In {\em international conference on machine learning}, pages
  2847--2854. PMLR, 2017.

\bibitem[STR18]{serra2018bounding}
Thiago Serra, Christian Tjandraatmadja, and Srikumar Ramalingam.
\newblock Bounding and counting linear regions of deep neural networks.
\newblock In {\em International Conference on Machine Learning}, pages
  4558--4566. PMLR, 2018.

\bibitem[SYK21]{sato2021random}
Ryoma Sato, Makoto Yamada, and Hisashi Kashima.
\newblock Random features strengthen graph neural networks.
\newblock In {\em Proceedings of the 2021 SIAM International Conference on Data
  Mining (SDM)}, pages 333--341. SIAM, 2021.

\bibitem[Tel16]{telgarsky2016}
Matus Telgarsky.
\newblock Benefits of depth in neural networks.
\newblock In Vitaly Feldman, Alexander Rakhlin, and Ohad Shamir, editors, {\em
  Proceedings of the 29th Conference on Learning Theory, {COLT} 2016, New York,
  USA, June 23-26, 2016}, volume~49 of {\em {JMLR} Workshop and Conference
  Proceedings}, pages 1517--1539. JMLR.org, 2016.

\bibitem[TRWG21]{toenshoff2021graph}
Jan Toenshoff, Martin Ritzert, Hinrikus Wolf, and Martin Grohe.
\newblock Graph learning with 1d convolutions on random walks.
\newblock {\em arXiv preprint arXiv:2102.08786}, 2021.

\bibitem[VLF20]{vignac2020building}
Clement Vignac, Andreas Loukas, and Pascal Frossard.
\newblock Building powerful and equivariant graph neural networks with
  structural message-passing.
\newblock {\em Advances in Neural Information Processing Systems},
  33:14143--14155, 2020.

\bibitem[WKK{\etalchar{+}}20]{wieder2020compact}
Oliver Wieder, Stefan Kohlbacher, M{\'e}laine Kuenemann, Arthur Garon, Pierre
  Ducrot, Thomas Seidel, and Thierry Langer.
\newblock A compact review of molecular property prediction with graph neural
  networks.
\newblock {\em Drug Discovery Today: Technologies}, 37:1--12, 2020.

\bibitem[WL68]{wlpaper}
Boris Weisfeiler and Andrew Lehman.
\newblock A reduction of a graph to a canonical form and an algebra arising
  during this reduction.
\newblock {\em Nauchno-Technicheskaya Informatsiya, Ser. 2, no. 9 (1968), 12-16
  (in Russian)}, 1968.

\bibitem[XHLJ19]{xu2019}
Keyulu Xu, Weihua Hu, Jure Leskovec, and Stefanie Jegelka.
\newblock How powerful are graph neural networks?
\newblock In {\em 7th International Conference on Learning Representations,
  {ICLR} 2019, New Orleans, LA, USA, May 6-9, 2019}. OpenReview.net, 2019.

\bibitem[YHC{\etalchar{+}}18]{ying2018graph}
Rex Ying, Ruining He, Kaifeng Chen, Pong Eksombatchai, William~L Hamilton, and
  Jure Leskovec.
\newblock Graph convolutional neural networks for web-scale recommender
  systems.
\newblock In {\em Proceedings of the 24th ACM SIGKDD international conference
  on knowledge discovery \& data mining}, pages 974--983, 2018.

\bibitem[YYZ18]{yu2018spatio}
Bing Yu, Haoteng Yin, and Zhanxing Zhu.
\newblock Spatio-temporal graph convolutional networks: a deep learning
  framework for traffic forecasting.
\newblock In {\em Proceedings of the 27th International Joint Conference on
  Artificial Intelligence}, pages 3634--3640, 2018.

\end{thebibliography}



\appendix

\section{Omitted Proofs of Section \ref{sec:intro}}\label{sec:omitted_intro}

\begin{proposition}\label{prop:union}
    Let $\ell_v^{(T)}$ be the labels of nodes after running the WL test for $T$ iterations on an input graph $G$ with initial labels $\ell_v^{(0)}$ all equal and suppose $p \le \delta/(n^2T)$. Suppose $T$ independent GNN iterations are successful according to Definition \ref{def:one_iteration} and the output labels of iteration $i$ are the input labels of iteration $i+1$ for all $i$ with the initial labels being equal to $\ell_v^{(0)}$. Let $h_v^{(T)}$ denote the node labels which are the output of the final iteration of the GNN. Then for all $v_1, v_2 \in V$, the following statements hold with probability $1-\delta$:
    \begin{itemize}
        \item If $\ell_{v}^{(T)} = \ell_u^{(T)}$ then $h_{v}^{(T)} = h_{u}^{(T)}$ and
        \item If $\ell_{v}^{(T)} \ne \ell_u^{(T)}$ then $h_{v}^{(T)} \ne h_{u}^{(T)}$.
    \end{itemize}
\end{proposition}
\begin{proof}
Consider some iteration $i \le T$. Suppose we have the following guarantee on the node label inputs for the $i$th iteration (note the inputs are the output labels of the previous iteration):
    \begin{itemize}
        \item If $\ell_{v}^{(i-1)} = \ell_u^{(i-1)}$ then $h_{v}^{(i-1)} = h_{u}^{(i-1)}$ and
        \item If $\ell_{v}^{(i-1)} \ne \ell_u^{(i-1)}$ then $h_{v}^{(i-1)} \ne h_{u}^{(i-1)}$.
    \end{itemize}
Given the guarantee for iteration $i-1$, we show that the same guarantee is true for the labels outputted in the $i$th iteration with probability $1-p$. Consider what the WL test does given input labels $\ell_{v_1}^{(i-1)}$: it assigns the same node labels to all pairs of vertices with the same multiset of neighborhood (input) labels and different labels for all pairs of vertices with different multiset of neighborhood (input) labels. This is exactly the same guarantee as Definition \ref{def:one_iteration}, except we also have to union bound over all $\le n^2$ pairs of nodes in Definition \ref{def:one_iteration}. Thus, the invariant stated in the beginning of the proof also holds for the $i$th iteration with probability $1-n^2p$. Now applying a union bound across all $T$ iterations, the total failure probability is $n^2Tp \le \delta$, as desired.
\end{proof}

\section{Omitted Proofs of Section \ref{sec:first_construction}}\label{sec:omitted_proofs_ub1}

\begin{proof}[Proof of Lemma \ref{lem:first_construction}]
Since all vectors $h_u^{(k)}$ are always one-hot encoded vectors, if two nodes have differing neighborhood multisets of labels, then the sum of the one-hot encoding vectors of each neighborhoods will also be different. Suppose we are in this case and let $y$ and $y'$ denote the sum of the neighborhood labels of $v$ and $u$ respectively. Without loss of generality, suppose $y$ and $y'$ differ in the first coordinate, $y_1$ and $y_1'$. Then $\langle a , y-y' \rangle \equiv 0 \bmod F$ if and only if $a_1(y_1 - y_1') \equiv c \bmod F$ where $c = \sum_{i=2}^F a_i(y_i' -  y_i)$. Condition on the event that $a_1 \not \equiv 0 \bmod F$ which happens with probability $1-1/F$. We now claim that $y_1' - y_1 \not \equiv 0 \bmod{F}$. Indeed, $0 \le y_1 \ne y_1' \le n$ so $ 0 \ne |y_1 - y_1'| \le n$. Since $F > 2n$, it cannot be the case that $y_1 - y_1' \equiv 0 \bmod F$ which means that the multiplicative inverse of $(y_1 - y_1') \bmod F$, denoted as $(y_1 - y_1')^{-1}$, is well-defined and unique. Thus for $a_1(y_1 - y_1') \equiv c \bmod F$, we must have $a_1 \equiv c \dot (y_1 - y_1')^{-1} \bmod F$ which happens with probability $O(1/F)$. Altogether, we have that  $\BP_a[\langle a , y\rangle \equiv \langle a, y' \rangle \bmod F] \le O(1/F)$,
as desired. Conditioning on $\langle a , y\rangle \not \equiv \langle a, y' \rangle \bmod F$, we have that $h_v^{(k)} \ne h_u^{(k)}$ as their single non-zero coordinates are distinct. Finally, we can easily check that if the multiset of neighborhoods of $u$ and $v$ are the same, then $h_v^{(k)} = h_u^{(k)}$ always holds.
\end{proof}

\subsection{Implementation of Modulo $F$ via ReLU Network}\label{sec:modF}

We now give an efficient ReLU network construction of the function computing modulo $F$, thereby proving Theorem \ref{thm:modF}. First we define the function $\TW_F: \{0, \ldots, NF\} \rightarrow \{0, \ldots, F-1\}$, also known as the ``triangular wave" function:
\[\TW_F(x) = \mathrm{T}_F(x \bmod{2F})\] where 
\[\mathrm{T}_F(x) = 
\begin{cases}
        x, \text{ if }  x \le F \\ 
        2F-x, \text{ if } x > F.
\end{cases} \]

A result of \cite{telgarsky2016} implements the function $\TW_F$ using low-depth ReLU networks.

\begin{theorem}[Lemma 3.8 and Corollary 3.9 in \cite{telgarsky2016}] Suppose $F = \poly(n)$. The function $\TW_F: \{0, \ldots, nF \} \rightarrow \{0, \ldots, F-1\}$ can be implemented using a neural network with $O(\log n)$ hidden units and $O( \log n)$ depth.
\end{theorem}

We now reduce the modulo $F$ case to the construction of the triangular wave function.

\modF*

\begin{proof}
    Assume without loss of generality that $F$ is odd. Given an integer $z \in \Z$, we first compute $\TW_F(z)$ and $\TW_{F/2}(z)$. Note that even though $F/2$ is non-integral, we can still compute it. Looking for a general real number $z \in [0, 2F]$, we have that $\TW_F(z) = \TW_{F/2}(z)$ if and only if $z \in [0, F/2]$ or $z \in [3F/2, 2F]$. In addition, if $z$ is an integer in $[0, 2F]$ but not in $[0, F/2] \cup [3F/2, 2F]$, then $\TW_F(z)-\TW_{F/2}(z) \ge 1$. Noting that $\text{ReLU}(1-\text{ReLU}(1-x)) = 0$ for $x = 0$ and $\text{ReLU}(1-\text{ReLU}(1-x)) = 1$ for any $x \ge 1,$ we therefore have that 
    \[\text{ReLU}(1-\text{ReLU}(1-\TW_F(z)+\TW_{F/2}(z)))\]
    equals $0$ if $z \in \Z$ and $0 \le (z \bmod{2F}) \le (F-1)/2$ or $(3F+1)/2 \le (z \bmod{2F}) \le 2F-1$, and equals $1$ if $z \in \Z$ and $(F+1)/2 \le (z \bmod{2F}) \le (3F-1)/2$. So, if we consider shifting $z$ by $(F-1)/2$, we can make this equal $1$ if and only if $z$ is between $F$ and $2F-1$ mod $F$. Therefore, for integers $z$,
    \begin{align*}
        &\text{ReLU}\left(1-\text{ReLU}\left(1-\TW_F\left(z-\frac{F-1}{2}\right)+\TW_{F/2}\left(z-\frac{F-1}{2}\right)\right)\right)  \\
        & = \begin{cases}0 & 0 \le (z \bmod{2F}) \le F-1 \\ 1 & F \le (z \bmod{2F}) \le 2F-1 \end{cases}.
    \end{align*}
    Now, for simplicity, let us define the function above as $g(z)$. Note that when $g(z) = 0$, then $\TW_F(z) = z \bmod{F}$, and when $g(z) = 1$, then $\TW_F(z) = F-(z \bmod{F})$, so $z\bmod{F} = F-\TW_F(z)$. Therefore, we have that for any integer $z$, $z \bmod{F} = |F \cdot g(z) - \TW_F(z)|$. But note that we can write 
    \[|x| = \max(x, -x) = \max(2x, 0) - x = \text{ReLU}(2x) - x.\] Therefore, we have that
\[z \bmod{F} = \text{ReLU}(2 F \cdot g(z) - 2 \cdot \TW_F(z)) - F \cdot g(z) + \TW_F(z). \qedhere\]
\end{proof}

\section{Omitted Proofs for Section \ref{sec:best_construction}}\label{sec:omitted_proofs_ub2}
\begin{proof}[Proof of Lemma \ref{lem:eps_biased}]
    Since $\langle a, x \rangle - \langle a, y \rangle = \langle a, x-y \rangle,$ by writing $z = x-y \in \Z^m$, it suffices to show that if $z$ is nonzero, then $\BP_{a \sim \mathcal{D}} [\langle a, z \rangle = 0] \le \frac{1}{2}+\frac{\eps}{2}$.
    
    Let $k$ represent the largest nonnegative integer such that $2^k|z_i$ for all $i \in [m] := \{1, 2, \dots, m\},$ and let $z' = z/2^k$. Then, by replacing $z$ with $z'$, we have that $\BP_{a \sim \mathcal{D}} [\langle a, z \rangle = 0]$ if and only if $\BP_{a \sim \mathcal{D}} [\langle a, z' \rangle = 0]$ (since we are just dividing by $2^k$) and $z'_i$ is odd for at least one value of $i \in [m]$. So, it suffices to show that $\BP_{a \sim \mathcal{D}} [\langle a, z' \rangle = 0] \le \frac{1}{2}+\frac{\eps}{2}$. Note that by reducing modulo $2$, it suffices to show that over $\GF_2,$ $\BP_{a \sim \mathcal{D}} [\langle a, z' \rangle \equiv 0 \bmod 2] \le \frac{1}{2}+\frac{\eps}{2}$, because then the probability over the integers is either the same or lower.
    
    Let $I$ be the subset of $i \in [m]$ for which $z'_i$ is odd. Note that $I$ is nonempty since we know $z'_i$ is odd for at least one value of $i \in [m]$. So, by the definition of $\eps$-biased sample spaces, we know that over $\GF_2$, $\left|2 \cdot \BP_{a \in \mathcal{D}} \left[\sum_{i \in I} a_i = 0\right] - 1\right| \le \eps$, which means that $\BP_{a \in \mathcal{D}} \left[\sum_{i \in I} a_i = 0\right] \le \frac{1}{2} + \frac{\eps}{2}$. But indeed $\sum_{i \in I} a_i \equiv \langle a, z'\rangle$ since $i \in I$ precisely when $z'_i$ is odd, so this completes the proof.
\end{proof}

\subsection{Proof of Corollary \ref{cor:eps_biased}}
We first need to define the circuit class TC$^0$.

\begin{definition}[Threshold Gate]\label{def:tc0}
For inputs $x_1, \ldots, x_m \in \{0,1\}$ the output of a threshold gate, $\textup{TH}$, is
\[\textup{TH}(x_1,\ldots, x_m) = \begin{cases}
  1  & \sum_{i=1}^m a_i x_i \ge \theta \\
  0 &  \text{ otherwise}
\end{cases}\]
where $\theta, a_1,\ldots,a_m \in \mathbb{Z}$. $\theta, a_1, \ldots, a_m$ may depend on $n$ but they do not depend on the input $x_1 , \ldots, x_m$.
\end{definition}

\begin{definition}[TC$^0$ Circuit Class]
\textup{TC}$^0$ is the class of boolean functions computed by constant-depth $\poly(m)$-size circuits with threshold gates.
\end{definition}

It is known that $\eps$-biased vectors can be generated using an efficient circuit in \textup{TC}$^0$.

\begin{theorem}[Theorem 14 in \cite{healy2006constant}, Restated]\label{thm:tc0}
Let $s = O(\log F + \log(1/\eps))$. For every $\eps$ and $F$, there exists an explicit \textup{TC}$^0$ circuit $C:\{0,1\}^s \cup \{0,1\}^{\lceil log_2 F \rceil} \rightarrow \{0,1\}$ which takes as input $s$ uniform random bits and an index $i \in [F]$ and outputs the $i$th coordinate of an $\eps$-biased vector in $\{0,1\}^F$. $C$ uses $\poly(s)$ threshold gates.
\end{theorem}

Note that the guarantees of Theorem \ref{thm:tc0} are not directly applicable since we need to use a ReLU network instead of threshold gates. Nevertheless, since the circuit $C$ guaranteed by Theorem \ref{thm:tc0} has integer inputs in all gates, we can easily approximate each threshold gates using an appropriately scaled ReLU. This is a straightforward and known reduction but we briefly outline a procedure in Lemma \ref{lmm:thredhold}.

\begin{lemma}\label{lmm:thredhold}
Consider the threshold gate \textup{TH}: $\{0,1\}^m \rightarrow \{0,1\}$ which computes the threshold $\sum_{i=1}^m a_i x_i \ge \theta$. Assume that $a_i, \theta$ are all integers bounded by $\poly(m)$. \textup{TH} can be computed by a ReLU network using $O(\log m)$ bits of precision and a constant number of parameters.
\end{lemma}
\begin{proof}
Consider the function 
\[g(x) = \text{ReLU}(-\text{ReLU}(-x+2)). \]
It is $0$ for all integers $x \le 0$ and $1$ for all integers $\ge 1$, i.e., it computes the threshold $``x \ge 0"$. By shifting and scaling $g$, we can now compute the threshold $``x \ge \theta"$ for any integer $\theta$. Finally, the sum $\sum_{i=1}^m a_i x_i$ can be computed using one additional layer. Since all parameters are integers, we only require $O(\log m)$ bits of precision to store the shifting and scaling factors.
\end{proof}

Lastly, we remark that as per the definition of a threshold gate in Definition \ref{def:tc0}, Theorem \ref{thm:tc0} requires the index $i \in [F]$ to be inputted as a binary string with its bits given on individual nodes. However, this presents a slight inconsistency with the statement of Theorem \ref{thm:tc0} and its corollary, Corollary \ref{cor:eps_biased} which is used in the construction of Section \ref{sec:best_construction}. Specifically, Step $3$ of the construction of Section \ref{sec:best_construction} outputs the actual integer $i \in [F]$ which we use as the index for our $\eps$-biased vector, which does not match the format required by Theorem \ref{thm:tc0}. This inconsistency is straightforward to fix without having any impact whatsoever in the asymptotic size complexity of the neural network. We simply take the integer $i$ outputted by Step $3$ of the construction and compute the $j$th bit of $i$ for all $1 \le j \le O(\log F)$ in parallel. The $j$th bit is exactly equal to $0$ if and only if $(i \bmod 2^{j+1})<2^j$ and $1$ otherwise. Note that $2^{j+1} = O(F)$ for all $j$ and we can easily compute each $\bmod 2^{j+1}$ by appealing to Theorem \ref{thm:modF}. This only requires $O(1)$ extra depth and an additional $O(\poly(\log n))$ hidden units and $O(\log n)$ bits of precision. The more general trade-off of Theorem \ref{thm:WL-GNN} also readily holds.

\section{Lower bounds}\label{sec:lower_bounds}
In this appendix we provide lower bounds on the complexity of graph neural networks that are able to simulate the WL test. We present both a communication complexity lower bound and a lower bound on the number of ReLU units of the GNN. More concretely, in Section~\ref{sec:low-bound-communication}, we prove that in order to maintain the invariant that  with at least some constant probability, nodes with isomorphic neighborhoods get the same label while nodes with non-isomorphic neighborhoods get different labels, some message sent between nodes must be of length at least $\Omega(\log n)$. This bound matches the upper bound of Theorem~\ref{thm:WL-GNN}.
Second, in Section~\ref{sec:low-bound-units}, we consider a more specific although still fairly general lower bound model which captures the implementation of the WL test using neural networks. We suppose that the messages sent between nodes are $t$-dimensional vectors with integral entries. We moreover suppose that each node combines its received messages by summing them to get a vector in $[F]^t$ (here, $[F]=\{0,1,\dots, F-1\}$) and applying a collectively agreed upon neural network $\phi$ with at most $H$ ReLU units to this sum.
We show that if the combination of summing neighborhoods and applying the neural network maps distinct multisets to distinct elements with at least some constant probability, then $H=\Omega(\log F)$. Moreover, parametrizing in terms of the depth and width of the neural network, we obtain a more fine-grained lower bound, demonstrating that for shallow neural networks, we need even more ReLU units.  In Remark~\ref{remark:NN-aggregate}, we point out that our lower bound holds even if the aggregation function $f$ is itself a neural network with a bounded number of ReLU units.
As a node in an $n$-node graph could have up to $n-1$ neighbours, we need at least $F=\Omega(n)$ in order to store the sum of the messages from the neighbors of the nodes. With this assumption, the lower bound thus becomes $\Omega(\log n)$ which matches our upper bound up to $\polylog (n)$  factors.  
It remains an interesting open problem to bridge the gap between the upper and lower bound. 

For both our lower bounds we assume that the nodes have access to an infinite public string of random bits. In Section~\ref{sec:low-bound-units}, this is the string which the nodes use to collectively agree on some neural network network $f$ with respect to some distribution on such networks with at most $H$ ReLU units.

\subsection{Lower Bound: Communication Complexity}\label{sec:low-bound-communication}

We consider a forest graph $G$ composed of pieces $G_1, G_1', G_2, G_2' \dots, G_m, G_m'$, for $m = \Theta(\sqrt{n})$. Each piece $G_k$ consists of a ``top'' node $u_k$, which is only connected to a ``middle'' node $v_k$, which in turn is connected to $k$ ``bottom'' nodes $w_{k, 1}, \dots, w_{k, k}$, and $G_k'$ is simply a duplicate of $G_k$ (with vertices $u_k', v_k'$, and $w_{k, j}'$ for $1 \le j \le k$). See Figure \ref{fig:forest_G} for a depiction of $G$. We note that after two rounds, each $u_k$ (and $u_k'$) should know the respective value of $k$, because the local graph of depth $2$ around $u_k$ is distinct for each $k \in [m]$.

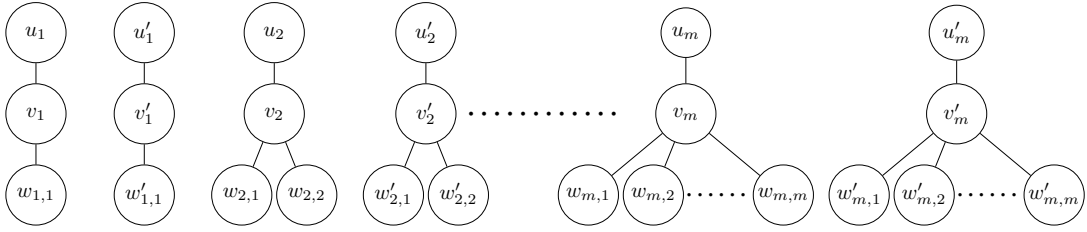
\begin{figure}
\begin{center}
\begin{tikzpicture}[scale=0.90]
\draw[white] (-0.5,-0.5) rectangle (15.5,3);
\begin{pgflowlevelscope}{\pgftransformscale{0.8}}

\node[draw, circle, minimum size=1cm] at (0, 3) (u1) {$u_1$};
\node[draw, circle, minimum size=1cm] at (0, 1.5) (v1) {$v_1$};
\node[draw, circle, minimum size=1cm, inner sep=0pt] at (0, 0) (w11) {$w_{1,1}$};

\draw[black] (u1) -- (v1);
\draw[black] (v1) -- (w11);

\node[draw, circle, minimum size=1cm] at (2, 3) (u1') {$u_1'$};
\node[draw, circle, minimum size=1cm] at (2, 1.5) (v1') {$v_1'$};
\node[draw, circle, minimum size=1cm, inner sep=0pt] at (2, 0) (w11') {$w_{1,1}'$};

\draw[black] (u1') -- (v1');
\draw[black] (v1') -- (w11');

\node[draw, circle, minimum size=1cm] at (4.4, 3) (u2) {$u_2$};
\node[draw, circle, minimum size=1cm] at (4.4, 1.5) (v2) {$v_2$};
\node[draw, circle, minimum size=1cm, inner sep=0pt] at (3.8, 0) (w21) {$w_{2,1}$};
\node[draw, circle, minimum size=1cm, inner sep=0pt] at (5.0, 0) (w22) {$w_{2,2}$};

\draw[black] (u2) -- (v2);
\draw[black] (v2) -- (w21);
\draw[black] (v2) -- (w22);

\node[draw, circle, minimum size=1cm] at (7.2, 3) (u2') {$u_2'$};
\node[draw, circle, minimum size=1cm] at (7.2, 1.5) (v2') {$v_2'$};
\node[draw, circle, minimum size=1cm, inner sep=0pt] at (6.6, 0) (w21') {$w_{2,1}'$};
\node[draw, circle, minimum size=1cm, inner sep=0pt] at (7.8, 0) (w22') {$w_{2,2}'$};

\draw[black] (u2') -- (v2');
\draw[black] (v2') -- (w21');
\draw[black] (v2') -- (w22');

\node at (9.4, 1.5) {$\large{\boldsymbol{\cdots \cdots \cdots \cdots}}$};

\node[draw, circle, minimum size=1cm, minimum size=0.5cm] at (12, 3) (um) {$u_m$};
\node[draw, circle, minimum size=1cm] at (12, 1.5) (vm) {$v_m$};
\node[draw, circle, minimum size=1cm, inner sep=0pt] at (10.2, 0) (wm1) {$w_{m,1}$};
\node[draw, circle, minimum size=1cm, inner sep=0pt] at (11.4, 0) (wm2) {$w_{m,2}$};
\node[draw, circle, minimum size=1cm, inner sep=0pt] at (13.8, 0) (wmm) {$w_{m,m}$};

\draw[black] (um) -- (vm);
\draw[black] (vm) -- (wm1);
\draw[black] (vm) -- (wm2);
\draw[black] (vm) -- (wmm);
\node at (12.6, 0) {$\boldsymbol{\cdots \cdots}$};

\node[draw, circle, minimum size=1cm, minimum size=0.5cm] at (17, 3) (um') {$u_m'$};
\node[draw, circle, minimum size=1cm] at (17, 1.5) (vm') {$v_m'$};
\node[draw, circle, minimum size=1cm, inner sep=0pt] at (15.2, 0) (wm1') {$w_{m,1}'$};
\node[draw, circle, minimum size=1cm, inner sep=0pt] at (16.4, 0) (wm2') {$w_{m,2}'$};
\node[draw, circle, minimum size=1cm, inner sep=0pt] at (18.8, 0) (wmm') {$w_{m,m}'$};

\draw[black] (um') -- (vm');
\draw[black] (vm') -- (wm1');
\draw[black] (vm') -- (wm2');
\draw[black] (vm') -- (wmm');
\node at (17.6, 0) {$\boldsymbol{\cdots \cdots}$};

\end{pgflowlevelscope}
\end{tikzpicture}
\end{center}
\caption{The graph $G$ used for Theorem \ref{thm:cc_lower}, showing the disconnected pieces $G_1, G_1', G_2, G_2', \dots, G_m, G_m'$. Note that for all $k$, $G_k$ and $G_k'$ are isomorphic.}
 \label{fig:forest_G}
\end{figure}

\begin{theorem} \label{thm:cc_lower}
    Suppose there exists a public random string $r$ that every node of $G$ has access to, and each node additionally has some independent private randomness. Suppose there is a communication protocol where by the end, with probability at least $3/4$, the following hold.
\begin{itemize}
    \item For every $k \in [m]$, the top nodes $u_k$ and $u_k'$ output the same value.
    \item For every $k \neq \ell \in [m]$, the top nodes $u_k, u_\ell$ output distinct values.
\end{itemize}
    Then, there must be some $k$ such that the edge $(u_k, v_k)$ or the edge $(u_k', v_k')$ has at least $\Omega(\log n)$ total bits of communication. Hence, if there are only $O(1)$ rounds of communication, one of those rounds must have sent $\Omega(\log n)$ bits of communication across the edge.
\end{theorem}

\begin{remark}
In comparison to Definition \ref{def:one_iteration}, we note that our lower bound works for $p = \frac{1}{4m^2} = \Theta\left(\frac{1}{n}\right)$ in Definition \ref{def:one_iteration}. This is because we require all nodes $u_k, u_\ell$ to simultaneously have different outputs with probability at least $\frac{3}{4},$ which is implied by a union bound if every $u_k, u_\ell$ for $k \neq \ell$ have different outputs with probability at least $1-\frac{1}{4m^2}$.
In addition, we actually prove a stronger lower bound against Definition \ref{def:one_iteration}, because our lower bound holds even if we allow each node to have its own independent private randomness,
and only requires nodes with the same local neighborhood to output the same answer simultaneously with probability $3/4$ instead of probability $1$.
\end{remark}

\begin{proof}
First, we note that we may assume the communication is one-way from $v_k$ to $u_k$. This is because the node $v_k$ can simulate all communication from $u_k$, as $u_k$ has no information about neighbors apart from $v_k$. So, we just need to show the one-way communication complexity is $\Omega(\log n)$.
Next, we will assume there is no public randomness - we will remove this assumption at the end.
So, each $u_k$ (resp., $u_k'$) receives at most $b$ bits of information from $v_k$ (resp., $v_k'$).
If $v_k$ sends a randomized message of length $b$ to $u_k$ and $u_k$ uses this message to produce some output $o_k$, with probability at least $\frac{3}{4}$ the outputs $o_k$ must all be distinct. In addition, for each $k$, the outputs of the duplicate copies of $u_k$ must be the same with probability at least $\frac{3}{4}$. Our goal is to show that $b = \Omega(\log n).$

Let $f$ be a randomized function from $[m]$ to $\{0, 1\}^{b}$, representing the randomized message $v_k$ sends to $u_k$ assuming it has full knowledge of its number of neighbors. 
Let $g$ be a randomized function from $\{0, 1\}^{b}$ to some arbitrary output space $\mathcal{O}$, which represents the final output of $u_k$ after it has seen the message from $v_k$.
Then, for all $k \neq \ell \in [m]$, $\BP[g_1(f_1(k)) \neq g_2(f_2(\ell))] \ge \frac{3}{4}$, but for all $k \in [m]$, $\BP[g_1(f_1(k)) = g_2(f_2(k))] \ge \frac{3}{4}$. Here, $g_1, g_2$ represent the function $g$ with different instantiations of the randomness, and $f_1, f_2$ represent the function $f$ with different instantiations of the randomness.

Fix some $k \in [m]$ and for each output $o \in \mathcal{O}$, let $p_k(o)$ represent the probability of outputting $g(f(k)) = o$. Note that $p_k(o)^2$ is the probability that $u_k$ outputs $o$ \emph{and} $u_k'$ outputs $o$, so the probability that $u_k$ and $u_k'$ have the same output equals $\sum_{o \in \mathcal{O}} p_k(o)^2$, which we are assuming is at least $\frac{3}{4}$. This means that $\max_{o \in \mathcal{O}} p_k(o) = \sum_{o \in \mathcal{O}} p_k(o) \cdot \max_{o \in \mathcal{O}} p_k(o) \ge \sum_{o \in \mathcal{O}} p_k(o)^2 \ge \frac{3}{4}.$ Therefore, for all $k \in [m]$, there exists an output $\tilde{o}_k$ such that $\BP[g(f(k)) = \tilde{o}_k] \ge \frac{3}{4}$.
Therefore, there must exist a value $s_k \in \{0, 1\}^{b}$ such that $\BP[g(s_k) = \tilde{o}_k] \ge \frac{3}{4}$. Indeed, if not, then for any distribution over $s_k \in \{0, 1\}^b$, we have that $\BP[g(s_k) = \tilde{o}_k] < \frac{3}{4}$, which means that $\BP[g(f(k)) = \tilde{o}_k] < \frac{3}{4}$.
In addition, $\tilde{o}_k$ is different across all $k$, as if $\tilde{o}_k = \tilde{o}_\ell$ for some $k \neq \ell$, then $\BP[g(f(k)) = g(f(\ell)) = \tilde{o}_k] \ge \left(\frac{3}{4}\right)^2 \ge \frac{1}{2}$, which means that  $\BP[g(f(k)) \neq g(f(\ell))] \le \frac{1}{2}$. 

But now, note that the $s_k$'s must be distinct, because if $s_k = s_{\ell}$, then because $\BP[g(s_k) = \tilde{o}_k] \ge \frac{3}{4}$ and $\BP[g(s_\ell) = \tilde{o}_\ell] \ge \frac{3}{4}$, this means that $\tilde{o}_k = \tilde{o}_\ell$.
Therefore, $s_1, \dots, s_m$ are all distinct. But since each $s_i$ lies in $\{0, 1\}^{b}$, this means that $b = \Omega(\log m) = \Omega(\log n)$, as desired.

To finish, we revisit the fact that we assumed there was no public randomness. Let us reintroduce the random string $r$ that every node of $G$ is given. We assume that with probability at least $3/4$, the top nodes $u_k, u_k'$ have the same output for all $k \in [m]$ and that the nodes $u_1, \dots, u_m$ output pairwise distinct values. But as this event happens with probability at least $3/4$ over a random string $r$, there must exist a choice of $r$ for which it happens with probability at least $3/4$ conditioned on $r$. But then we are back to the case where there is no public randomness, as desired.
\end{proof}

\subsection{Lower Bound: ReLU Units}\label{sec:low-bound-units}

\begin{figure}
\centering
\includegraphics{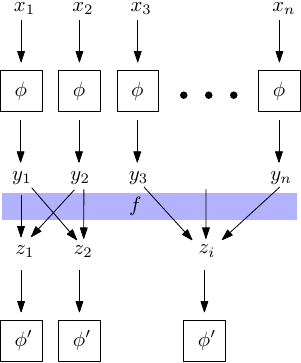}
\caption{The structure of the GNN in our lower bound model. For every node $i$, node $i$ receives the message $x_i$ which is the sum of the labels of the neighbors of $i$. The algorithm then picks a random $\phi\in \Phi_H$ with respect to $D$ and for each $i=1,\dots,n$, calculates $y_i=\phi(x_i)$ which is the new label of node $i$. Next, each node $i$, sends $y_i$ to each of its neighbors and every node $i$ then calculates $z_i=f(\{y_j\mid j \in \mathcal{N}(i)\})=\sum_{j\in \mathcal{N}(i)}y_i$. The $z_i$'s then form the node inputs at the next iteration to a new randomly chosen $\phi'\in \Phi_H$.}
\label{fig:NN-structure}
\end{figure}
We next prove a lower bound on the number of ReLU units that we need to implement the WL test as a graph neural network. To do so, we will use that any neural network $\phi: \R^t\to \R^\ell$ which uses ReLU's as activation functions induces a partition of $\R^t$ into convex polytopes $R_1,\dots,R_N$ such that $\phi$ restricted to each $R_i$ is just a linear function. Moreover, the number of such regions can be bounded using the following theorem.
\begin{theorem}[Proposition $4$ in \cite{montufar2014number}]\label{thm:linear-regions1}
The number of linear regions of any ReLU network $\phi:\R^t\to\R^d$ with a total of $H$ ReLU units is at most $2^H$.
\end{theorem}

Let $t,\ell,F\in \N$ be given and denote by $[F]=\{0,1,\dots,F-1\}$. Let $\Phi_H$ denote the family of neural networks $\phi: \R^t\to \R^{\ell}$ with at most $H$ ReLU units mapping. We will consider the following model (see Figure~\ref{fig:NN-structure}):
The algorithm designer picks a distribution $D$ over $\Phi_H$. 
Let $G$ be an arbitrary $n$-node graph and recall that $\mathcal{N}(i)$ denotes the neighborhood of node $i$ (including $i$ itself). For arbitrary inputs $x_1,\dots, x_n\in [F]^t$ to the $n$ nodes (which we think of as the sums of the messages that the nodes receive), the GNN operates as follows: Using the publicly available random string, the nodes collectively pick a random neural network $\phi\in \Phi_H$ with respect to $D$. For $i=1,\dots, n$, node $i$ then calculates $y_i=\phi(x_i)$ (which we think of as the new label of node $i$). Then each node $i$ sends $y_i$ to each node in its neighborhood $\mathcal{N}(i)$ and calculates $z_i=\sum_{j\in \mathcal{N}(i)}y_j$. We remark that this notation is different from what we introduced in Section~\ref{sec:prel} and~\ref{sec:tech-overview}. Referring back to those sections, $i$ corresponds to node $u$, $x_i$ corresponds to $\mathcal{S}_u^{(k-1)}$, i.e., the sum of the received messaged of node $u$ in iteration $k-1$, $\phi$ corresponds to $\phi^{(k)}$, $y_i$ corresponds to $h_u^{(k)}$, i.e., the new label of node $u$, and $z_i$ corresponds to $\mathcal{S}_u^{(k)}$, i.e., the sum of the received messaged of node $u$ in iteration $k$. We have made this switch in notation to make the argument that follows less unwieldy.

We would like our GNN to satisfy that for any $n$-node graph $G$, and arbitrary inputs $x_1,\dots,x_n$ to the nodes, it holds with probability at least $9/10$ over the randomness of $D$ that $z_i\neq z_j$ for all $i,j$ such that the multisets $\{x_k\mid k\in \mathcal{N}(i)\}$ and $\{x_k\mid k\in \mathcal{N}(j)\}$ are different. 
The following theorem provides a lower bound on the number of ReLU units $H$ needed for this property to hold.
\begin{theorem}\label{thm:hidden-units-lower}
Suppose that the neural networks in $\Phi_H$ have at most $H\leq \lg F-4$ ReLU units. Then there exists a graph $G$ on $n$ nodes and inputs $x_1,\dots,x_n\in [F]^t$ such that if $\mathcal{N}(1),\dots, \mathcal{N}(n)$ are the neighborhoods of the nodes of $G$, then with probability at least $1-\left(\frac{3}{4}\right)^{n/6}$, there exists $i,j\in[n]$ such that $z_i= z_j$ even though the multisets $\{x_k\mid k\in \mathcal{N}(i)\}$ and $\{x_k\mid k\in \mathcal{N}(j)\}$ are different.
Thus, to simulate the WL test with neural networks from $\Phi_{H}$, we need $H>\log F-3$ ReLU units.
\end{theorem}
Before proving the theorem, we first explain how to interpret it as a lower bound for the computational complexity of implementing a WL iteration as in Definition~\ref{def:one_iteration} as a neural network. As an initial observation, note that in any iteration $k$, if for two nodes $u$ and $v$, the sums $S_u^{(k-1)}$ and $S_v^{(k-1)}$ are distinct (recall that $\mathcal{S}^{(k-1)}_u = \sum_{u' \in \mathcal{N}(u)} h^{(k-1)}_u$), then for the WL iteration to be successful, we must also have that $h^{(k)}_u\neq h^{(k)}_v$. This is because, $S_u^{(k-1)}\neq S_v^{(k-1)}$ implies that the multisets $\{h_{u'}^{(k-1)}\mid u'\in \mathcal{N}(u)\}$ and $\{h_{v'}^{(k-1)}\mid v'\in \mathcal{N}(v)\}$ are also distinct. But then Definition~\ref{def:one_iteration} yields that we need $h_u^{(k)}\neq h_v^{(k)}$ (at least with some probability $1-p$).
Now, consider two nodes $i=u$ and $j=v$ such that in some iteration of the WL test, the multisets of sums $\{\mathcal{S}_{u'}^{(k-1)}\mid u'\in \mathcal{N}(u)\}$ and $\{\mathcal{S}_{v'}^{(k-1)}\mid v'\in \mathcal{N}(v)\}$ are different.
This corresponds to the multisets $\{x_k\mid k\in \mathcal{N}(i)\}$ and $\{x_k\mid k\in \mathcal{N}(j)\}$ being different. We would like to argue that for the WL iteration to be successful according to Definition~\ref{def:one_iteration}, for each such pair of nodes $u,v$, we must have that also $\mathcal{S}_u^{(k)}\neq \mathcal{S}_v^{(k)}$ with some good probability (the sums of labels in the next iteration differ). Theorem~\ref{thm:hidden-units-lower} tells us that the probability of this happening is very low if we use too few ReLU units. Now why do we require that $\mathcal{S}_u^{(k)}\neq \mathcal{S}_v^{(k)}$ for such a pair of nodes $u,v$?

Since the multisets $\{\mathcal{S}_{u'}^{(k-1)}\mid u'\in \mathcal{N}(u)\}$ and $\{\mathcal{S}_{v'}^{(k-1)}\mid v'\in \mathcal{N}(v)\}$ are different, by the initial observation, the multisets $\{h_{u'}^{(k)}\mid u'\in \mathcal{N}(u)\}$ and $\{h_{v'}^{(k)}\mid v'\in \mathcal{N}(v)\}$ must also be distinct for the WL test to be successful. 
But since the multisets of labels $\{h_{u'}^{(k)}\mid u'\in \mathcal{N}(u)\}$ and $\{h_{u'}^{(k)}\mid v'\in \mathcal{N}(v)\}$ are distinct it follows by another application of Definition~\ref{def:one_iteration}, that we must also have that $h_u^{(k+1)}\neq h_v^{(k+1)}$. However, the only way this can happen is if $\mathcal{S}_{u}^{(k)}\neq \mathcal{S}_{v}^{(k)}$ as otherwise these two sums will be mapped to the same label by $\phi^{(k+1)}$. 

We remark that this lower bound applies to an isolated WL iteration rather than a full sequence of iterations. In particular, the inputs $x_1,\dots,x_n\in [F]^t$ (corresponding to the sums $\{\mathcal{S}_{u}^{(k-1)}\mid u\in G\}$) are adversarially chosen while in reality these inputs are not arbitrary but are the result of a prior WL iteration. Our construction in Section~\ref{sec:best_construction} indeed works against such adversarially chosen sums in the sense that different multisets of sums are mapped (via applying $\phi^{(k)}$ and summing the outputs for each multisets) to different sums with high probability, and as such our lower bound is exactly a lower bound for this harder problem. However, in general the sums $S_v^{(k-1)}$ are not adversarially chosen, and it would very be interesting to find a lower bound that does not require this assumption but works all the way from a graph and its initial labels.

\begin{proof}[Proof of Theorem~\ref{thm:hidden-units-lower}]
We exhibit a graph $G$ and a distribution $D_0$ over the possible inputs, such that for any neural network $\phi:[F]^t\to \R^{\ell}$ with at most $H$ ReLU units, if  $(x_1,\dots, x_n)\in([F]^t)^n$ is chosen with respect to $D_0$, then  with probability at least $1-\left(\frac{3}{4}\right)^{n/6}$, there exists $i,j$ such that $z_i=z_j$ even though the multisets $\{x_k\mid k\in \mathcal{N}(i)\}$ and $\{x_k\mid k\in \mathcal{N}(j)\}$ are different. It then follows from Yao's minimax principle that for any distribution $D$ over $\Phi_H$, there exists an input $x=(x_1,\dots,x_n)$ such that the same bad event occurs with the same high probability. As desired. 

\begin{figure}
\centering
\includegraphics[scale=.8]{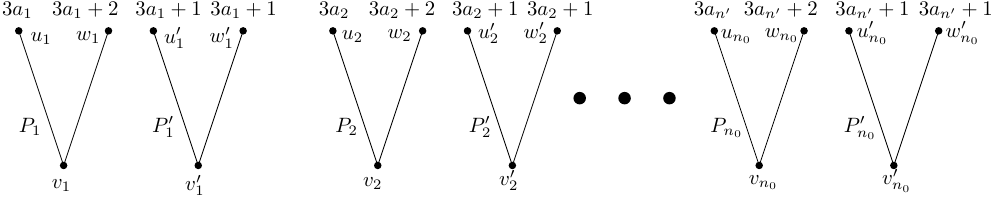}
\caption{The graph of our lower bound construction of Theorem~\ref{thm:hidden-units-lower}. It consists of $n/3$ copies of a path of length $3$ grouped into pairs of two. For each such pair of paths $(P,P')$, the two end nodes of $P$ are assigned (random) inputs (which we think of as the sum of their received messages) that are guaranteed to have the same sum as the sum of the inputs to the end nodes of $P'$. The idea of the proof is that if the neural network $\phi$ has too few ReLU unit, then this linear dependence is preserved with probability $\Omega(1)$ even after applying $\phi$.}
\label{fig:lower-bound-construction}
\end{figure}

Let us start out by describing the graph $G$. Assume with no loss of generality that $6$ divides $n$, i.e.~$n=6n_0$ for some natural number $n_0$. The graph $G$ simply consists of $2n_0$ copies of a path of length $2$ grouped into pairs $(P_1,P_1'),\dots,(P_{n_0},P_{n_0}')$ (see Figure~\ref{fig:lower-bound-construction}). Denote the vertices of path $P_i$ by $(u_i,v_i,w_i)$ and similarly the vertices of path $P_i'$ by $(u_i',v_i',w_i')$. We next proceed to describe the distribution $D_0$ over inputs $(x_1,\dots, x_n)\in([F]^t)^n$. Here, for each $i< n_0$, we let $x_{3i},x_{3i+1},x_{3i+2}$ be the inputs to nodes $u_i$, $v_i$ and $w_i$ respectively and $x_{3i+3},x_{3i+4},x_{3i+5}$ be the inputs to nodes $u_i'$, $v_i'$ and $w_i'$. To do so, let $\ell=\lfloor F/3\rfloor$ and define for each $0\leq a< \ell$ and $b\in [F]^{t-1}$ the set $F_{a,b}=\{(3a,b),(3a+1,b),(3a+2,b)\}$. Thus, $F_{a,b}\subseteq [F]^t$ consists of the three vectors which have their first coordinate equal to respectively $3a$,$3a+1$ and $3a+2$, and $b$ as their last $t-1$ coordinates. 
For each $i< n_0$, we independently pick uniformly random $0< a\leq \ell$ and $b\in [F]^{t-1}$. We then define $x_{6j}=(3a,b)$, $x_{6j+2}=(3a+2,b)$, and $x_{6j+3}=x_{6j+5}=(3a+1,b)$ (see Figure~\ref{fig:lower-bound-construction}).
We further define $x_{6j+1}=x_{6j+4}=0$ (the exact value is unimportant, as long as they are equal).  Note that $x_{6j}+x_{6j+1}+x_{6j+2}=x_{6j+3}+x_{6j+4}+x_{6j+5}$. In particular, if the random $\phi\in \mathcal F$ was a linear map, we would have that \begin{align}\label{eq:if-linear}
z_{6j+1}=\phi(x_{6j})+\phi(x_{6j+1})+\phi(x_{6j+2})=\phi(x_{6j+3})+\phi(x_{6j+4})+\phi(x_{6j+5})=z_{6j+4},
\end{align} in spite of the multisets, $\{x_{6j},x_{6j+1},x_{6j+2}\}$ and $\{x_{6j+3},x_{6j+4},x_{6j+5}\}$ being different. Now, $\phi$ is a neural network, so this identity does not need to hold. The idea is however, that if $\phi$ has only few ReLU units, then we obtain a good upper bound on the number of \emph{linear regions} by Theorem~\ref{thm:linear-regions1} and since $a$ and $b$ are random, the set $F_{a,b}$ is likely to be fully contained in one of these regions. And since $\phi$ restricted to this region in linear,~\eqref{eq:if-linear} holds in this case. 

To formalize this, assume that $\phi:[F]^t\to \R^{\ell}$ is a neural network in $\Phi_H$ with at most $H$ ReLU units. Using Theorem~\ref{thm:linear-regions1}, it follows that $[0,F-1]^t$ can then be partitioned into at most $2^H$ convex regions $R_1,\dots,R_N$, 
such that for each region $R_i$, it holds that $\phi$ restricted to $R_i$ is simply a linear function. For $0\leq a< \ell$ and $b\in [F]^{t-1}$, we let $L_{a,b}$ be the straight line segment connecting the two points $(3a,b)$ and $(3a+2,b)$. By convexity, for a fixed $R_i$ and a fixed $b\in [F]^{t-1}$, at most $2$ of the line segments in $\{L_{a,b}\mid 0\leq a<\ell\}$ can intersect the boundary of $R_i$. Since there are $F^{t-1}$ distinct choices for $b$ and $N\leq 2^H$ choices for $i$, this implies that the number of line segments $L_{a,b}$ that can cross \emph{any} of the boundaries of the convex regions is at most $2^{H+1}F^{t-1}$. Each of the remaining segments $L_{a,b}$ must be fully contained in some region $R_i$. There are $\lfloor F/3 \rfloor F^{t-1}$ choices of $a$ and $b$ in total, and thus, at least $\lfloor F/3 \rfloor F^{t-1}-2^{H+1}F^{t-1}$ of the segments $L_{a,b}$ 
are fully contained in one of the regions $R_i$, which means that $\phi$ restricted to $L_{a,b}$ acts as a linear map.
For such a segment $L_{a,b}$, we get by linearity that 
$$
\phi(3a,b)+\phi(3a+2,b)=2\phi(3a+1,b).
$$
In other words, if for a fixed $i\leq n_0$, $(a,b)$ is chosen such that $L_{a,b}$ is fully contained in one of the convex regions, then ~\eqref{eq:if-linear} is satisfied. 
It follows that for any given $i<n_0$,
\begin{align}\label{eq:regions}
\Pr[z_{6i+1}=z_{6i+4}]\geq 1-\frac{2^{H+1}F^{t-1}}{\lfloor F/3 \rfloor F^{t-1}}\geq 1- \frac{12\cdot 2^{H}}{F}.
\end{align}
Since the event $(z_{6i+1}=z_{6i+4})_{i\leq n_0}$ are independent (as we choose independent $(a,b)$ for each pair of paths $(P_i,P_i')$), it follows that
$$
\Pr[\exists 0 \leq i\leq n_0: z_{6i+1}=z_{6i+4}]\geq 1-\left(\frac{12\cdot 2^H}{F}\right)^{n_0}.
$$
If in particular, $H\leq \lg F-4$, we obtain that 
$$
\Pr[\exists 0 \leq i\leq n_0: z_{6i+1}=z_{6i+4}]\geq 1-\left(\frac{3}{4}\right)^{n/6},
$$
As the multisets $\{x_k\mid k\in \mathcal{N}(6i+1)\}$ and $\{x_k\mid k\in \mathcal{N}(6i+4)\}$ are different, this completes the proof.
\end{proof}
\begin{remark}\label{remark:NN-aggregate}
In analogue with our construction in Section~\ref{sec:first_construction} and Section~\ref{sec:best_construction}, we assumed in the above proof that the aggregate function $f$ is the summation function. As such, $f$ is just another neural network but without a single ReLU unit. We can therefore think of the combined computation performed by $f$ and $\phi$ (illustrated in Figure~\ref{fig:lower-bound-construction}) as the result of applying a single neural network. It follows from this observation (and the proof
of Theorem~\ref{thm:hidden-units-lower}) that in the more general setting where $f$ is a function in $\Phi_{H_1}$ and where the neural network $\phi\in\Phi_{H_2}$, then we must have that $H_1+H_2> \log F-4$  in order to successfully simulate the WL test.
\end{remark}

\subsubsection{Better bounds on the number of linear regions.}\label{sec:more-units-results}

The proof of Theorem~\ref{thm:hidden-units-lower} used that the number of convex linear regions of any neural network with at most $H$ ReLU's is at most $2^H$. However, in many cases one can obtain better upper bounds on the number of such regions, and this directly translates to a better lower bound than the one given in Theorem~\ref{thm:hidden-units-lower}. Indeed, if the family of neural networks $\Phi$ satisfies that the domain of any $\phi\in \Phi$ can be partitioned into at most $K$ convex regions such that $\phi$ restricted to each of these regions is linear, then the lower bound in \eqref{eq:regions} instead becomes
$$
\Pr[z_{6i+1}=z_{6i+4}]\geq 1- \frac{12\cdot K}{F}.
$$
In particular, when using independence of the events $(z_{6i+1}=z_{6i+4})_{i< n_0}$, we just need $K\leq \frac{F}{24}$, say, to get that an error occur with probability at least $1-2^{-\Omega(n)}$. Plugging in the bound $K\leq 2^H$ of Theorem~\ref{thm:linear-regions1} gave the desired bound of Theorem~\ref{thm:hidden-units-lower} which led to the $\Omega(\log F)$ lower bound. If we instead use the more fine grained theorem below, we obtain better bounds for shallow neural networks with low input dimension as stated in Corollary~\ref{cor:low-hidden-units2}. 

\begin{theorem}[Theorem 1 in \cite{raghu2017expressive}]\label{thm:linear-regions2}
Any ReLU neural network with input dimension $t$, width $w$, and depth $d$ has at most $O(w^{d \cdot t})$ linear regions.
\end{theorem}
\begin{corollary}\label{cor:low-hidden-units2}
Let $\Phi_{t,d,w}$ consist of all ReLU neural networks with input dimension $t$, depth $d$, and width $w$. Suppose that we are in the setting of Theorem~\ref{thm:hidden-units-lower}, except that the neural network is picked from $\Phi_{t,d,w}$. 
Then the conclusion of the theorem holds as long as $w^{d \cdot t}\leq c F$ for a small enough constant $c$. In particular, to simulate the WL test with neural networks from $\Phi_{t,d,w}$, we need $dw=\Omega (dF^{\frac{1}{dt}})$ ReLU units.
\end{corollary}
As an example, for shallow neural networks with low input dimension, say with $d,t=O(1)$, this lower bound becomes $dw=F^{\Omega(1)}$, i.e. polynomial rather than logarithmic in the size of the underlying field. 

For certain architectures of the neural networks one can obtain even stronger bounds on the number of linear regions (see e.g., Proposition 3 in~\cite{montufar2017notes} and Theorem 1 in~\cite{serra2018bounding}). These bounds are parametrized in the number of ReLU units in each of the $d$ layers of the neural networks. One can therefore obtain even more fine grained lower bounds on the number of ReLU units if one makes more assumptions on the family of neural networks but the bounds are more opaque and we refrain from stating them here. 

\subsubsection{Description complexity}
In this subsection, we consider more general function classes $\Phi$ that do not necessarily have to consist of neural networks. We prove that for any aggregation function $f$, if for any two distinct multisets of labels each of size at most $n$, there exists a function $\phi \in \Phi$ such that the multisets are mapped to different labels by $\phi\circ f$, then $|\Phi|=\Omega\left(\frac{n}{\log n}\right)$. It follows that the description complexity of $\Phi$ must be $\Omega(\log n)$. Our bound is combinatorial, and does not employ the linear structure of $F^t$. Hence we may just put $t=1$ and think of $F$ as a set rather than a vector space.

\begin{theorem}\label{thm:description-complexity}
Let $F$ and $n$ be natural numbers and let $N=\sum_{i=0}^n\binom{i+F-1}{i}$ be the number of multisets of $[F]$ of size at most $n$. Suppose that $f$ is any aggregation function mapping multisets of $[F]$ to some range $\mathcal{R}$. Let $\Phi$ be any set of functions from $\mathcal{R}$ to $F$ and assume that  $|\Phi|< \frac{\log N}{\log F}$. Then there exists distinct multisets $A$ and $B$ each with at most $n$ elements from $[F]$ such that $\phi(f(A))=\phi(f(B))$ for all $\phi \in \Phi$. 
\end{theorem}

\begin{proof}
Note that each function $\phi\in \Phi$ induces a partition $\mathcal{P}_\phi$ on the set of these multisets induced by the equivalence relation defined by $X\sim_{\phi} Y \Longleftrightarrow \phi(f(X))=\phi(f(Y))$. By repeated application of the pidgeonhole principle, there must exists a collection $\mathcal{C}$ of multisets each containing at most $n$ elements from $[F]$ such that (1) for all $\phi\in \Phi$ and all $X,Y\in \mathcal{C}$, $X\sim_{\phi} Y$ and (2) $|\mathcal{C}|\geq N/F^{|\Phi|}$. If in particular $|\Phi|< \frac{\log N}{\log F}$, we must have that $|\mathcal{C}|\geq 2$. Letting $A$ and $B$ be distinct elements of $\mathcal{C}$, we have that $\phi(f(A))=\phi(f(B))$ for all $\phi \in \Phi$.
\end{proof}

Note that number of distinct degrees of the vertices of a simple $n$-node graph could be as large as $\Omega(n)$, so it is a natural assumption that also $F=\Omega(n)$. Indeed, if $F$ is smaller, then for such a graph, the simulation of the WL test will fail with probability $1$ since it must inevitably assign two nodes of distinct degrees the same label in $F$. With this assumption, it follows that 
\begin{align*}
    \log N&=\log \sum_{i=0}^n\binom{i+F-1}{i}\geq \log  \binom{n+F-1}{n}\geq \log\left(1+\frac{F-1}{n}\right)^n\\
    &=\Omega\left(n\cdot \max\left\{1,\log \frac{F}{n}\right\}\right),
\end{align*}

using the inequality $\binom{n}{k}\geq \left(\frac{n}{k}\right)^k$. Thus,
$$
\frac{\log N}{\log F}=\Omega\left(n \cdot \frac{\max\left\{1,\log \frac{F}{n}\right\}}{\log F} \right)=\Omega\left(\frac{n}{\log n} \right).
$$
In particular, the description complexity of $\Phi$ has to be $\Omega(\log n)$ in order to separate any two distinct multisets $X,Y\subset F$ each consisting of at most $n$ elements.
We note that the description complexity of the construction in Section~\ref{sec:best_construction} is $\poly \log n$. 

\end{document}